\documentclass[leqno,12pt]{amsart} 
\usepackage{amssymb}
\usepackage{mathrsfs}
\usepackage{mathtools}

\DeclareMathAlphabet{\mathscrbf}{OMS}{mdugm}{b}{n}
\DeclareMathOperator{\spt}{spt}
\DeclareMathOperator{\diam}{diam}

\renewcommand{\epsilon}{\varepsilon}
\renewcommand{\phi}{\varphi}
\newcommand{\defeq}{\mathrel{\mathop:}=}

\theoremstyle{plain} 
\newtheorem{theorem}{\indent\sc Theorem}[section]
\newtheorem{lemma}[theorem]{\indent\sc Lemma}
\newtheorem{corollary}[theorem]{\indent\sc Corollary}
\newtheorem{proposition}[theorem]{\indent\sc Proposition}

\theoremstyle{definition} 
\newtheorem{definition}[theorem]{\indent\sc Definition}
\newtheorem{remark}[theorem]{\indent\sc Remark}

\usepackage{algorithmic}
\usepackage{booktabs}
\usepackage{appendix}
\usepackage{dsfont}
\usepackage{ellipsis}
\usepackage{float}
\usepackage{paralist}
\usepackage[inline]{enumitem}
\usepackage{amsbsy}

\usepackage{algorithm}
\usepackage{listings}
\lstdefinelanguage{PseudoCode}{
  alsoletter={-,?,/,*},
  basewidth=0.5em,
  classoffset=0,
  morekeywords={
    loop,with,recur,let,end,define,if,
    nil,then,else,return,for,while,do,
    read,and,where,input,output,procedure,
    forever,or,exit,forall,being,for,each,
    descending,assert,where,not,false,True,any,fpsep,
    first_terminal_scc,sccalg,combine,rcombine,
    choose,layered,rlayered,False,compute,in
  },
  keywordstyle=\bfseries,
  classoffset=1,
  morekeywords={
    next-closed,model-gcis,empty?,make-ordered-set,
    add-elements,hasse-graph,seq,add-element,
    explore-attributes
  },
  keywordstyle=\upshape\ttfamily,
  classoffset=0,
  mathescape=true,
  texcl=true,
  commentstyle=\rm,
  numbers=left,
  stepnumber=1,
  firstnumber=0,
  numberstyle={\tiny},
  escapeinside={(*@}{@*)},
  morecomment=[l]{;},
  commentstyle={\itshape},
  columns=fullflexible,
}
\usepackage[backend=bibtex,natbib=true,style=numeric,doi=false,isbn=false,giveninits=true,
url=false]{biblatex}
\usepackage[hidelinks]{hyperref}

\begin{document}

\title[Intrinsic Dimension of Geometric Data Sets]{Intrinsic Dimension of Geometric Data Sets} 

\makeatletter
\def\author@andify{%
  \nxandlist {\unskip ,\penalty-1 \space\ignorespaces}%
    {\unskip {} \@@and~}%
    {\unskip \penalty-2 \space \@@and~}%
}
\makeatother

\author[T. Hanika]{Tom Hanika} 

\author[F.M. Schneider]{Friedrich Martin Schneider$^*$} 
\author[G. Stumme]{Gerd Stumme} 


%
\thanks{ 
$^{*}$F.M.S.~acknowledges
    funding of the Excellence Initiative by the German Federal and State
    Governments, as well as the Brazilian CNPq, processo 150929/2017-0.
}
\address{
  Knowledge \& Data Engineering Group\endgraf
  University of Kassel\endgraf  
  34121 Kassel\endgraf
  Germany
}
\email{tom.hanika@cs.uni-kassel.de}

\address{
Institute of Discrete Mathematics and Algebra\endgraf
TU Bergakademie Freiberg\endgraf
09596 Freiberg\endgraf
Germany
}
\email{martin.schneider@math.tu-freiberg.de}

\address{
  Knowledge \& Data Engineering Group\endgraf
  University of Kassel\endgraf  
  34121 Kassel\endgraf
  Germany
}
\email{stumme@cs.uni-kassel.de}


\maketitle

\begin{abstract}
  The curse of dimensionality is a phenomenon frequently observed in machine
  learning (ML) and knowledge discovery (KD). There is a large body of
  literature investigating its origin and impact, using methods from mathematics
  as well as from computer science. Among the mathematical insights into data
  dimensionality, there is an intimate link between the dimension curse and the
  phenomenon of measure concentration, which makes the former accessible to
  methods of geometric analysis. The present work provides a comprehensive study
  of the intrinsic geometry of a data set, based on Gromov's metric measure
  geometry and Pestov's axiomatic approach to intrinsic dimension. In detail, we
  define a concept of geometric data set and introduce a metric as well as a
  partial order on the set of isomorphism classes of such data sets. Based on
  these objects, we propose and investigate an axiomatic approach to the
  intrinsic dimension of geometric data sets and establish a concrete dimension
  function with the desired properties. Our  model for data sets and
  their intrinsic dimension is computationally feasible and, moreover, adaptable
  to specific ML/KD-algorithms, as illustrated by various experiments.
\end{abstract}

\section*{Introduction} 
\label{intro}
One of the essential challenges in data driven research is to cope with sparse
and high dimensional data sets. Various machine learning (ML) and knowledge
discovery (KD) procedures are susceptible to the so-called \emph{curse of
  dimensionality}.  Despite its frequent occurrence, this effect lacks for a
comprehensive computational approach to decide if and to what extent a data set
will be tapped with it. Pestov’s work~\citep{Pestov0} revealed that the dimension
curse is closely linked to the phenomenon of \emph{concentration of measure},
which was discovered itself by~\citet{Milman1988,Milman2010,Milman1983}
and is also known as the Lévy property. This link enables the study of the
dimension curse through methods of geometric analysis.

A valuable step towards an indicative for concentration is the axiomatic
approach for an \emph{intrinsic dimension} of data
by~\citet{Pestov0,Pestov1,Pestov2}, which involves modeling data sets as metric
spaces with measures and utilizing geometric analysis for their quantitative
assessment. His work is based on Gromov's \emph{observable distance} between
metric measure spaces~\citep[Chapter~3$\tfrac{1}{2}$.H]{Gromov99} and uses
observable invariants to define concrete instances of dimension
functions. However, despite its mathematical elegance, this approach is
computationally infeasible, as discussed in~\citep[Section~IV]{Pestov1}
and~\citep[Sections~5, 8]{Pestov2}, because it amounts to computing the set of
all real-valued $1$-Lipschitz functions on a metric space. Pestov suggests a way
out~\citep[Section 8]{Pestov1} by considering a data set as a pair $(X,F)$
consisting of a metric measure space $X$ together with a set $F\subseteq
\mathrm{Lip}_{1}(X)$ of computationally cheap feature functions, e.g., distance
functions to points~\citep[Section~IV]{Pestov1}.

In the present paper, we build up on this idea and demonstrate a geometric
model that is both theoretically comprehensive and computationally accessible. More precisely, we introduce the
notion of a~\emph{geometric data set} (Definition~\ref{defi:ds}), which may be regarded as metric measure space
together with a generating set of 1-Lipschitz functions, called
\emph{features}. The elements of the feature set are supposed to be both
computationally feasible and adaptable to the representation of data as well as
to the respective ML or KD procedure.  Upon constructing a specific metric on the set of isomorphism
classes of such geometric data sets (see Definition~\ref{definition:data.distance} and Theorem~\ref{theorem:observable.metric}), detecting the dimension curse
amounts to computing the distance of a geometric data set to the trivial (i.e.,
singleton) data set -- a problem related to the task in~\citet{blumberg}
where the authors determine tests to distinguish finite samples drawn from
different measures on a metric space through applying Gromov's $mm$-reconstruction
theorem. Furthermore, we propose on the class of geometric data sets
a revised version of Pestov's axiomatic system, i.e., a conception of a
\emph{dimension function} (Definition~\ref{definition:dimension.function}), and establish a concrete instance of such a dimension
function through adapting Gromov's notion of \emph{observable diameters} to the
geometric data sets (Proposition~\ref{proposition:dimension}).

For a first illustration of our approach, and in order to nourish our
understanding of the novel dimension function, we apply it to examples from two
essentially different domains: data sets in $\mathbb{R}^{n}$ and data sets
resembling incidence structures. For the former we provide an algorithm for
computing the intrinsic dimension function and show how the resulting values
behave for various artificial and real-world data sets. We investigate this in
particular in contrast to the intrinsic dimension due to~\citet{Chavez}. For
the latter case we show how to represent incidence structure as geometric data
set of the above kind and how to calculate their intrinsic dimension.  We
conclude our work by computing and discussing the intrinsic dimension for
several real-world data sets. Our computational results suggest that the
intrinsic dimension, as introduced in this work, does carry information not
captured by other invariants of data sets.

The present article is structured as follows. The preliminary Section~\ref{sec:geomlip}
is concerned with recollecting some basics of metric
geometry. In Section~\ref{sec:mmg}, we recall some bits of Gromov's seminal work on
observable geometry of metric measure spaces. The
subsequent Section~\ref{section:concentration} is dedicated to introducing our concept
of geometric data sets as well as defining and investigating a natural metric
and partial order on the collection of isomorphism classes of such. This is
followed by the adaptation of Gromov's observable diameters to our setting
in Section~\ref{sec:obsdiam}. In Section~\ref{section:intrinsic.dimension}, we then turn to
the study of dimension functions on geometric data sets. Subsequently, we apply
our results to two different use cases
in Sections~\ref{sec:dens-based-clust} and~\ref{sec:conc-clust} and conclude our work
with Section~\ref{sec:conc}.


\section{Geometry of Lipschitz functions}
\label{sec:geomlip}
The purpose of this section is to provide some background on the structure of
the set of $1$-Lipschitz functions on a metric space. Most importantly, this
will include a review of recent work by~\citet{BenYaacov}, see Proposition~\ref{proposition:lipschitz} below.

To begin with, let us fix some basic notation. Let $\mathscr{X} = (X,d)$ be a pseudo-metric space. The \emph{diameter} of $\mathscr{X}$ is defined as $\diam (\mathscr{X}) \defeq \sup \{ d(x,y) \mid x,y \in X \}$. Given any real number $\ell \geq 0$, we may consider the set \begin{displaymath}
	\left. \mathrm{Lip}_{\ell}(\mathscr{X}) \, \defeq \, \left\{ f \in \mathbb{R}^{X} \, \right| \forall x,y \in X \colon \, \vert f(x) - f(y) \vert \leq \ell d(x,y) \right\}
\end{displaymath} of all \emph{$\ell$-Lipschitz} real-valued functions on $\mathscr{X}$, and let \begin{displaymath}
	\mathrm{Lip}_{\ell}^{s}(\mathscr{X}) \, \defeq \, \{ f \in \mathrm{Lip}_{\ell}(\mathscr{X}) \mid \Vert f \Vert_{\infty} \leq s \}
\end{displaymath} for any real $s \geq 0$. For $x \in A \subseteq X$ and $\epsilon > 0$, we let $B_{d}(x,\epsilon) \defeq \{ y \in X \mid d(x,y) < \epsilon \}$ and $B_{d}(A,\epsilon) \defeq \{ y \in X \mid \exists a \in A \colon \, d(a,y) < \epsilon \}$. The \emph{Hausdorff distance} of two sets $A,B \subseteq X$ with respect to $d$ is given by \begin{displaymath}
	d_{\mathrm{H}}(A,B) \, \defeq \, \inf \{ \epsilon > 0 \mid B \subseteq B_{d}(A,\epsilon), \, A \subseteq B_{d}(B,\epsilon) \} .
\end{displaymath}

Now let $X$ be a set and let $F \subseteq \mathbb{R}^{X}$. We define $d_{F} \colon X \times X \to [0,\infty]$ by \begin{displaymath}
	d_{F}(x,y) \coloneqq \sup \{ \vert f(x) - f(y) \vert \mid f \in F \} \qquad (x,y \in X) .
\end{displaymath} We will call $F$ \emph{tame} if $d_{F}(x,y) < \infty$ for all
$x,y \in X$, in which case $d_{F}$ constitutes a pseudo-metric on $X$. Evidently, in
case $F$ is tame, $d_{F}$ is a metric on $X$ if and only if $F$ separates the
points of $X$, in the sense that $X \to \mathbb{R}^{F}, \, x \mapsto (f(x))_{f
  \in F}$ is injective. In the following, we aim to determine the set of
$1$-Lipschitz functions for $d_{F}$, i.e., to give an algebraic representation
of the elements of~$\mathrm{Lip}_{1}(X,d_{F})$ as generated from members of
$F$. We provide such a description in Proposition~\ref{proposition:lipschitz},
adapting work of~\citet{BenYaacov}.

Preparing the statement of Proposition~\ref{proposition:lipschitz}, let us introduce some additional notation. Given a set $M$, denote by $\mathscr{P}(M)$ the power set of $M$ and by $\mathscr{P}_{\mathrm{fin}}(M)$ the set of all finite subsets of $M$. Let $X$ be a set. For any finite non-empty subset $F \subseteq \mathbb{R}^{X}$, we obtain functions $\bigvee F, \, \bigwedge F \in \mathbb{R}^{X}$ defined by \begin{displaymath}
	\left(\bigvee F\right)\!(x) \, \defeq \, \max \{ f(x) \mid f \in F \} , \quad \, \left(\bigwedge F\right)\!(x) \, \defeq \, \min \{ f(x) \mid f \in F \} \qquad (x \in X) \, .
\end{displaymath} For any $n \in \mathbb{N}_{\geq 1}$ and $f_{1},\ldots,f_{n} \in \mathbb{R}^{X}$, we let \begin{displaymath}
	\bigvee\nolimits_{i=1}^{n} f_{i} \, \defeq \, \bigvee \{ f_{i} \mid i \in \{ 1,\ldots,n \} \} , \qquad \bigwedge\nolimits_{i=1}^{n} f_{i} \, \defeq \, \bigwedge \{ f_{i} \mid i \in \{ 1,\ldots,n \} \} \, .
\end{displaymath} Consider the closure operators $\mathscr{K}, \mathscr{L} \colon \mathscr{P}\!\left(\mathbb{R}^{X}\right) \to \mathscr{P}\!\left(\mathbb{R}^{X}\right)$ defined by \begin{displaymath}
	\mathscr{K}(F) \defeq \{ \alpha f + c \mid f \in F \cup \{ 0 \} , \, \alpha \in [-1,1], \, c \in \mathbb{R} \} \qquad \left(F \subseteq \mathbb{R}^{X}\right)
\end{displaymath} and \begin{displaymath}
	\left. \mathscr{L}(F) \defeq \left\{ \bigvee\nolimits_{i=1}^{n} \bigwedge F_{i} \, \right| n \in \mathbb{N}_{\geq 1} , \, F_{1},\ldots,F_{n} \in \mathscr{P}_{\mathrm{fin}}(F) \setminus \{ \emptyset \} \right\} \qquad \left(F \subseteq \mathbb{R}^{X}\right) .
\end{displaymath} Whereas the closure system associated to $\mathscr{L}$ is the
set of sublattices of $\mathbb{R}^{X}$, the closure system associated to
$\mathscr{K}$ is precisely the collection of all balanced subsets of the
$\mathbb{R}$-vector space $\mathbb{R}^{X}$ being moreover closed under
translations by constant functions. It is straightforward to prove that
$\mathscr{K}(\mathscr{L}(F)) \subseteq \mathscr{L}(\mathscr{K}(F))$ for every $F
\subseteq \mathbb{R}^{X}$, which readily implies that $\mathscr{L} \circ
\mathscr{K}$ constitutes a closure operator on $\mathbb{R}^{X}$, too. The
following result is a variation on work of~\citet{BenYaacov}

\begin{proposition}[cf.~{\cite[Theorem~4.3]{BenYaacov}}]\label{proposition:lipschitz} Let $X$ be a set and let $F \subseteq \mathbb{R}^{X}$ be tame. Then \begin{displaymath}
	\mathrm{Lip}_{1}(X,d_{F}) \, = \, \overline{\mathscr{L}(\mathscr{K}(F))} ,
\end{displaymath} where the (third) closure refers to the topology of pointwise
convergence on $\mathbb{R}^{X}$. \end{proposition}
\begin{proof} ($\supseteq$) Clearly, $F \subseteq \mathrm{Lip}_{1}(X,d_{F})$. It is easy to check that the set $\mathrm{Lip}_{1}(X,d_{F})$ is closed with respect to the operators $\mathscr{K}$ and $\mathscr{L}$ as well as the topology of pointwise convergence on $\mathbb{R}^{X}$, whence $\overline{\mathscr{L}(\mathscr{K}(F))}$ is contained in $\mathrm{Lip}_{1}(X,d_{F})$.
	
($\subseteq$) Let us first prove the following auxiliary statement.
	
\textit{Claim~$(\ast)$.} For all $\epsilon > 0$, $x,y \in X$ and $s,t \in \mathbb{R}$ with $\vert s-t\vert \leq d_{F}(x,y)$, there is $f \in \mathscr{K}(F)$ such that $\max \{ \vert s -f(x) \vert, \vert t - f(y) \vert \} \leq \epsilon$.
	
\textit{Proof of~$(\ast)$.} Let $\epsilon > 0$ and let $x,y \in X$, $s,t \in \mathbb{R}$ such that $\vert s-t \vert \leq d_{F}(x,y)$. Clearly, if $\vert s-t \vert \leq \epsilon$, then the desired conclusion follows from the fact that $\mathscr{K}(F)$ contains all constant functions. Thus, without loss of generality, we may and will assume~that $\vert s-t \vert > \epsilon$. By definition of $d_{F}$, there is $f \in F \cup (-F)$ with $\vert s-t \vert - \epsilon < f(x) - f(y) $. Considering \begin{displaymath}
	\alpha \defeq \tfrac{s-t -\epsilon}{f(x)-f(y)} \in (-1,1) 
\end{displaymath} and $c \defeq t - \alpha f(y)$, we observe that $g \defeq \alpha f + c \in \mathscr{K}(F)$, and moreover $g(y) = t$ and \begin{displaymath}
	g(x) - g(y) \, = \, \alpha (f(x) - f(y)) \, = \, s-t -\epsilon ,
\end{displaymath} so that $g(x) = s-\epsilon$. Hence, $\max \{ \vert s -g(x) \vert, \vert t - g(y) \vert \} \leq \epsilon$ as desired. \quad $\boxed{\ast}$
	
To prove that $\mathscr{L}(\mathscr{K}(F))$ is dense in $\mathrm{Lip}_{1}(X,d_{F})$, let $f \in \mathrm{Lip}_{1}(X,d_{F})$. Consider $\epsilon > 0$ and a non-empty finite subset $E \subseteq X$. By Claim~$(\ast)$, for each pair $(x,y) \in E^{2}$ there exists $f_{x,y} \in \mathscr{K}(F)$ such that \begin{displaymath}
	\max \{ \vert f(x) -f_{x,y}(x) \vert, \vert f(y) - f_{x,y}(y) \vert \} \leq \epsilon ,
\end{displaymath} whence $f_{x,y}(x) \leq f(x) + \epsilon$ and $f_{x,y}(y) \geq f(y) - \epsilon$ in particular. For each $x \in E$, it follows that \begin{displaymath}
	f_{x} \, \defeq \, \bigvee\nolimits_{y \in E} f_{x,y} \, \in \, \mathscr{L}(\mathscr{K}(F)) ,
\end{displaymath} while $f_{x}(x) \leq f(x) + \epsilon$ and $f_{x}(y) \geq f_{x,y}(y) \geq f(y) - \epsilon$ for all $y \in E$. Similarly, we observe that \begin{displaymath}
	g \, \defeq \, \bigwedge\nolimits_{x \in E} f_{x} \, \in \, \mathscr{L}(\mathscr{K}(F)) ,
\end{displaymath} and $g(x) \leq f_{x}(x) \leq f(x) + \epsilon$ as well as $g(x) \geq f(x) - \epsilon$ for every $x \in E$. That is, $\sup_{x \in E} \vert f(x) - g(x) \vert \leq \epsilon$. This shows that $\mathscr{L}(\mathscr{K}(F))$ is dense in $\mathrm{Lip}_{1}(X,d_{F})$. \end{proof}

\section{Metric Measure Spaces, Concentration, and Lipschitz
  Order}\label{sec:mmg} In this section, we recollect some pieces of metric
measure geometry, i.e., the theory of metric measure spaces. Most importantly,
this will include the concepts of \emph{observable distance}
(Definition~\ref{definition:observable.distance}) and \emph{Lipschitz order}
(Definition~\ref{definition:lipschitz.order}), introduced by~\citet{Gromov99}.

For a start, let us clarify some general measure-theoretic notation. Let $\mu$ be a probability measure on a measurable space $S$. Given another measurable space $T$, the \emph{push-forward measure} $f_{\ast}(\mu)$ of $\mu$ with respect to a measurable map $f \colon S \to T$ is the measure $f_{\ast}(\mu)$ on $T$ defined by $f_{\ast}(\mu) (B) \defeq \mu (f^{-1}(B))$ for every measurable $B \subseteq T$. For any measurable $T \subseteq S$ with $\mu (T) > 0$, the probability measure $\mu\! \! \upharpoonright_{T}$ on the induced measure space $T$ is given by $(\mu \!\!\upharpoonright_{T} )(B) \defeq \mu(T)^{-1}\mu (B)$ for every measurable $B \subseteq T$. Moreover, we obtain a pseudo-metric $\mathrm{me}_{\mu}$ on the set of all measurable real-valued functions on $S$ defined by \begin{displaymath}
	\mathrm{me}_{\mu}(f,g) \, \defeq \, \inf \{ \epsilon \geq 0 \mid \mu (\{ s \in S \mid \vert f(s) - g(s) \vert > \epsilon \}) \leq \epsilon \} 
\end{displaymath} for any two measurable $f,g \colon S \to \mathbb{R}$. When considering measures on topological spaces, we will moreover use the following concept: if $\gamma$ is a Borel probability measure on a Hausdorff space $X$, then the \emph{support} of $\gamma$ is defined as \begin{displaymath}
	\spt \gamma \, \defeq \, \{ x \in X \mid \forall U \subseteq X \text{ open} \colon \, x \in U \Longrightarrow \gamma (U) > 0\} ,
\end{displaymath} which constitutes a closed subset of $X$. Finally, we will denote by $\nu_{F}$ the normalized counting measure on a finite non-empty set $F$, i.e., $\nu_{F}(B) \defeq \vert F \vert^{-1} \vert B \vert$ for $B \subseteq F$.

\begin{definition}[metric measure space] A \emph{metric measure space}, or simply \emph{$mm$-space}, is a triple $\mathscr{X} = (X,d,\mu)$ consisting of a separable complete metric space $(X,d)$ and a probability measure $\mu$ on the Borel $\sigma$-algebra of $(X,d)$ with $\spt \mu = X$. Two $mm$-spaces $\mathscr{X}_{i} = (X_{i},d_{i},\mu_{i})$ $(i \in \{ 0,1 \})$ are called \emph{isomorphic}, and we write $\mathscr{X}_{0} \cong \mathscr{X}_{1}$, if there exists an isometric bijection $\phi \colon (X_{0},d_{0}) \to (X_{1},d_{1})$ such that $\phi_{\ast}(\mu_{0}) = \mu_{1}$. The set all isomorphism classes of $mm$-spaces will be denoted by $\mathscrbf{M}$. \end{definition}

Let us note the following fact about spaces of Lipschitz functions on $mm$-spaces.

\begin{lemma}\label{lemma:compact.lipschitz.spaces} Let $(X,d,\mu)$ be an $mm$-space and $k \in \mathbb{N}$. The topology on $\mathrm{Lip}_{1}^{k}(X,d)$ generated by $\mathrm{me}_{\mu}$ coincides with the topology of point-wise convergence. In particular, $\left( \mathrm{Lip}_{1}^{k}(X,d), \mathrm{me}_{\mu} \right)$ is a compact metric space. \end{lemma}

\begin{remark}\label{remark:compact.lipschitz.spaces} For any metric space $(X,d)$, the topology of point-wise convergence and the topology of uniform convergence on compact subsets coincide on $\mathrm{Lip}_{1}(X,d)$. \end{remark}

\begin{proof}[Proof of Lemma~\ref{lemma:compact.lipschitz.spaces}] Since $\spt \mu = X$, the map $\mathrm{me}_{\mu}$ constitutes a metric on $\mathrm{Lip}_{1}(X,d)$, hence on $\mathrm{Lip}_{1}^{k}(X,d)$. We invoke the well-known Arzel\`{a}-Ascoli theorem, as stated in~\citet[7.15, pp.~232]{kelley}: being an equicontinuous, compact subset of the product space $\mathbb{R}^{X}$, the set $\mathrm{Lip}_{1}^{k}(X,d)$ is compact with respect to the topology $\tau_{C}$ of uniform convergence on compact subsets of~$X$. We show that the topology $\tau_{M}$ generated by the metric $\mathrm{me}_{\mu}$ on $\mathrm{Lip}_{1}^{k}(X,d)$ is contained in~$\tau_{C}$. To this end, let $U \in \tau_{M}$ and consider any $f \in U$. Since $U \in \tau_{M}$, we find some $\epsilon > 0$ with $\left. \left\{ g \in \mathrm{Lip}_{1}^{k}(X,d) \, \right| \mathrm{me}_{\mu}(f,g) < \epsilon \right\} \subseteq U$. As $\mu$ is a Borel probability measure on the Polish space $X$, there exists a compact subset $K \subseteq X$ with $\mu (K) > 1-\epsilon$ (see, e.g.,~\citet[Chapter~II, Theorem~3.2]{part}). Consequently, \begin{displaymath}
	\left. \left\{ g \in \mathrm{Lip}_{1}^{k}(X,d) \, \right| \sup\nolimits_{x \in K} \vert f(x) - g(x) \vert < \epsilon \right\} \, \subseteq \, U ,
\end{displaymath} which entails that $U$ is a neighborhood of $f$ in $\tau_{C}$. This shows that $U \in \tau_{C}$. Thus, $\tau_{M} \subseteq \tau_{C}$ as desired. Since $\tau_{M}$ is Hausdorff and $\tau_{C}$ is compact, it follows that $\tau_{M} = \tau_{C}$. In the light of Remark~\ref{remark:compact.lipschitz.spaces}, this completes the proof. \end{proof}

Our next objective is to recollect Misha Gromov's notion for an \emph{observable distance}~\citep[Chapter~3$\tfrac{1}{2}$.H]{Gromov99} on $\mathscr{M}$. Let us recall the well-known fact that every Borel probability measure $\mu$ on a Polish space $X$ admits a \emph{parametrization}, that is, a Borel map $\phi \colon I \to X$ such that $\mu = \phi_{\ast}(\lambda)$ for the Lebesgue measure $\lambda$ on $I \defeq [0,1)$ see, e.g.,~\citet[Lemma~4.2]{ShioyaBook}. This justifies the following definition.

\begin{definition}\label{definition:observable.distance} The \emph{observable distance} between two $mm$-spaces $\mathscr{X}$ and $\mathscr{Y}$ is defined to be \begin{align*}
    d_{\mathrm{conc}}(\mathscr{X},\mathscr{Y}) \defeq \inf \{ (\mathrm{me}_{\lambda})_{\mathrm{H}}(\mathrm{Lip}_{1}(\mathscr{X}) \circ \phi, \mathrm{Lip}_{1}(\mathscr{Y}) \circ \psi) \mid\ &\phi \text{ param.~of } \mathscr{X}\!,\\  &\psi \text{ param.~of } \mathscr{Y}\}.\end{align*}

  A sequence of $mm$-spaces $(\mathscr{X}_{n})_{n \in \mathbb{N}}$ is said to \emph{concentrate to} an $mm$-space $\mathscr{X}$ if\linebreak $d_{\mathrm{conc}}(\mathscr{X}_{n},\mathscr{X}) \longrightarrow 0$ as $n \to \infty$. \end{definition}

It is straightforward to check that the observable distance is invariant under isomorphisms of
\mbox{$mm$-spaces},~i.e., $d_{\mathrm{conc}}(\mathscr{X}_{0},\mathscr{X}_{1}) =
d_{\mathrm{conc}}(\mathscr{Y}_{0},\mathscr{Y}_{1})$ for any two pairs of
isomorphic $mm$-spaces $\mathscr{X}_{i} \cong \mathscr{Y}_{i}$ ($i \in \{ 0,1
\}$). Furthermore, as proved by~\citet{Gromov99}, see also~\citet[Theorem~5.13]{ShioyaBook}, the map $d_{\mathrm{conc}}$ constitutes a metric on the set ${\mathscrbf{M}}$. We refer to the induced topology on ${\mathscrbf{M}}$ as the \emph{concentration topology}.

In addition to the observable distance, let us recall another tool of Gromov's
metric measure geometry, see~\citet{Gromov99} and also~\citet[Section~2.2]{ShioyaBook}.

\begin{definition}[Lipschitz order]\label{definition:lipschitz.order} Let $\mathscr{X}_{i} = (X_{i},d_{i},\mu_{i})$ $(i \in \{ 0,1 \})$ be a pair of $mm$-spaces. We say that $\mathscr{X}_{1}$ \emph{Lipschitz dominates} $\mathscr{X}_{0}$ and write $\mathscr{X}_{0} \preceq \mathscr{X}_{1}$ if there exists a $1$-Lipschitz map $\phi \colon (X_{1},d_{1}) \to (X_{0},d_{0})$ such that $\phi_{\ast}(\mu_{1}) = \mu_{0}$. \end{definition}

Since, for any two pairs of isomorphic $mm$-spaces $\mathscr{X}_{i} \cong \mathscr{Y}_{i}$ ($i \in \{ 0,1 \}$), \begin{displaymath}
	\mathscr{X}_{0} \preceq \mathscr{Y}_{0} \quad \Longleftrightarrow \quad \mathscr{X}_{1} \preceq \mathscr{Y}_{1} ,
\end{displaymath} one may consider $\preceq$ as a relation on $\mathscrbf{M}$,
which is then called \emph{Lipschitz order} on~$\mathscrbf{M}$. The
Lipschitz order constitutes a partial order on the set $\mathscrbf{M}$
see~\citet[Proposition~2.11]{ShioyaBook}. The proof of this fact given by~\citet[Section~2.2]{ShioyaBook} reveals the following.

\begin{lemma}\label{lemma:lipschitz.order} If $\mathscr{X}_{i} = (X_{i},d_{i},\mu_{i})$ $(i \in \{ 0,1 \})$ are $mm$-spaces with $\mathscr{X}_{1} \preceq \mathscr{X}_{0}$, then every $1$-Lipschitz map $\phi \colon (X_{1},d_{1}) \to (X_{0},d_{0})$ with $\phi_{\ast}(\mu_{1}) = \mu_{0}$ is an isometric bijection. \end{lemma}

\begin{proof} This is shown by~\citet[Proof of Lemma~2.12]{ShioyaBook}. \end{proof}

\section{Geometric Data Sets, Concentration, and Feature Order}\label{section:concentration}

In this section we propose a mathematical model for data sets (Definition~\ref{defi:ds}),
which is accessible to methods of geometric analysis. Subsequently, we introduce
and study a specific metric on the set of isomorphism classes of such data sets
(Definition~\ref{definition:data.distance}), as well as a natural partial order
(Definition~\ref{definition:feature.order}), both analogous to their respective
predecessors for metric measure spaces established by~\citet{Gromov99}. 

\begin{definition}[geometric data set]\label{defi:ds} A \emph{geometric data set} is a triple $\mathscr{D} = (X,F,\mu)$ consisting of a set $X$ equipped with a tame set $F \subseteq \mathbb{R}^{X}$ such that $(X,d_{F})$ is a separable complete metric space and a probability measure $\mu$ on the Borel $\sigma$-algebra of $(X,d_{F})$ with $\spt \mu = X$. Given a geometric data set $\mathscr{D} = (X,F,\mu)$, we will refer to the elements of $F$ as the \emph{features} of $\mathscr{D}$. Two geometric data sets $\mathscr{D}_{i} = (X_{i},F_{i},\mu_{i})$ $(i \in \{ 0,1 \})$ will be called \emph{isomorphic} and we will write $\mathscr{D}_{0} \cong \mathscr{D}_{1}$ if there exists a bijection $\phi \colon X_{0} \to X_{1}$ such that $\overline{F_{1}} \circ \phi = \overline{F_{0}}$ (where the closure operators refer to the respective topologies of point-wise convergence) and $\phi_{\ast}(\mu_{0}) = \mu_{1}$. The collection of all isomorphism classes of geometric data sets shall be denoted by ${\mathscrbf{D}}$. \end{definition}

We observe that ${\mathscrbf{D}}$ indeed constitutes a set, since any separable metric space has cardinality less than or equal to $2^{\aleph_{0}}$. Henceforth, we shall not distinguish between geometric data sets and isomorphism classes of such, that is, elements of ${\mathscrbf{D}}$. Alternatively to Definition~\ref{defi:ds}, one may think of a geometric data set as a \emph{marked} $mm$-space, i.e., a quadruple $(X,d,\mu,F)$ consisting of an $mm$-space~$(X,d,\mu)$ along with a subset $F \subseteq \mathrm{Lip}_{1}(X,d)$ such that $\mathrm{Lip}_{1}(X,d) = \overline{\mathscr{L}(\mathscr{K}(F))}$. This perspective is due to Proposition~\ref{proposition:lipschitz}. Of course, there are (at least) two kinds of geometric data sets naturally associated with every $mm$-space.

\begin{definition}[induced data sets]\label{definition:induced.data.sets} For any $mm$-space $\mathscr{X} = (X,d,\mu)$, we define \begin{align*}
	& \mathscr{X}_{\bullet} \coloneqq (X, \mathrm{Lip}_{1}(X,d), \mu ) , & \mathscr{X}_{\circ} & \coloneqq (X, \{ x \mapsto d(x,y) \mid y \in X \}, \mu ) .
\end{align*} \end{definition}

For a given $mm$-space, the two associated geometric data sets defined above may
differ drastically from each other, e.g., with respect to measure
concentration. As remarked by~\citet[pp.~188--189]{Gromov99}: ``For many examples, such as round spheres~$S^{n}$ and other symmetric spaces, the concentration of the distance function is child's play compared to that for all Lipschitz functions~$f$. But if we look at more general spaces, say homogeneous, non-symmetric ones, or manifold $X^{n}$ with $\mathrm{Ricci} \, X^{n} \geq n$, then establishing the concentration for the distance functions becomes a respectable enterprise.''

Seizing an idea by Pestov, we will study the following adaptation of Gromov's observable distance~\citep[see][Chapter~3$\tfrac{1}{2}$.H]{Gromov99} to our setup of data sets.

\begin{definition}[observable distance]\label{definition:data.distance} The \emph{observable distance} between two geometric data sets $\mathscr{D}_{0} = (X_{0},F_{0},\mu_{0})$ and $\mathscr{D}_{1} = (X_{1},F_{1},\mu_{1})$ is defined as \begin{displaymath}
	d_{\mathrm{conc}}(\mathscr{D}_{0},\mathscr{D}_{1}) \, \defeq \, \inf \{ (\mathrm{me}_{\lambda})_{\mathrm{H}}(F_{0} \circ \phi_{0},F_{1} \circ \phi_{1} ) \mid \phi_{0} \text{ param.~of } \mu_{0}, \, \phi_{1} \text{ param.~of } \mu_{1} \} .
\end{displaymath} \end{definition}

It is not difficult to see that $d_{\mathrm{conc}}$ is invariant under isomorphisms of geometric data sets, in the sense that $d_{\mathrm{conc}}(\mathscr{D}_{0},\mathscr{D}_{1}) = d_{\mathrm{conc}}(\mathscr{D}_{0}',\mathscr{D}_{1}')$ for any two pairs of isomorphic geometric data sets $\mathscr{D}_{i} \cong \mathscr{D}_{i}'$ $(i \in \{ 0,1 \})$. Henceforth, we will identify $d_{\mathrm{conc}}$ with the induced function on ${\mathscrbf{D}}^{2}$. This map constitutes a metric, as recorded in Theorem~\ref{theorem:observable.metric}. Before going into the specifics of Theorem~\ref{theorem:observable.metric} and its proof, let us furthermore introduce an analogue of the Lipschitz order (Definition~\ref{definition:lipschitz.order}) for geometric data sets.

\begin{definition}[feature order]\label{definition:feature.order} Let $\mathscr{D}_{i} = (X_{i},F_{i},\mu_{i})$ $(i \in \{ 0,1 \})$ be two geometric data sets. We say that $\mathscr{D}_{1}$ \emph{feature dominates} $\mathscr{D}_{0}$ and write $\mathscr{D}_{0} \preceq \mathscr{D}_{1}$ if there exists a map $\phi \colon X_{1} \to X_{0}$ such that $F_{0} \circ \phi \subseteq \overline{F_{1}}$ and $\phi_{\ast}(\mu_{1}) = \mu_{0}$. \end{definition}

Analogously with the situation for $mm$-spaces, if $\mathscr{D}_{i} \cong \mathscr{D}'_{i}$ ($i \in \{ 0,1 \}$) are any two pairs of isomorphic geometric data sets, then \begin{displaymath}
	\mathscr{D}_{0} \, \preceq \, \mathscr{D}_{1}  \quad \Longleftrightarrow \quad  \mathscr{D}'_{0} \, \preceq \, \mathscr{D}'_{1} \, .
\end{displaymath} Henceforth, we will identify $\preceq$ with the corresponding relation thus induced on $\mathscrbf{D}$ and call it the \emph{feature order} on $\mathscrbf{D}$.

\begin{proposition}\label{proposition:feature.order} $\preceq$ constitutes a partial order on $\mathscrbf{D}$. \end{proposition}

\begin{proof} Evidently, $\preceq$ is reflexive and transitive. To prove that $\preceq$ is anti-symmetric, let $\mathscr{D}_{i} = (D_{i},F_{i},\mu_{i})$ $(i \in \{ 0,1\})$ be two geometric data sets, and suppose that both $\mathscr{D}_{0} \preceq \mathscr{D}_{1}$ and $\mathscr{D}_{1} \preceq \mathscr{D}_{0}$. Then there exist maps $\phi \colon X_{0} \to X_{1}$ and $\psi \colon X_{1} \to X_{0}$ such that $F_{1} \circ \phi \subseteq \overline{F_{0}}$, $F_{0} \circ \psi \subseteq \overline{F_{1}}$, $\phi_{\ast}(\mu_{0}) = \mu_{1}$, and $\psi_{\ast}(\mu_{1}) = \mu_{0}$. Let $d_{0} \defeq d_{F_{0}}$ and $d_{1} \defeq d_{F_{1}}$, and observe that $\phi \colon (X_{0},d_{0}) \to (X_{1},d_{1})$ and $\psi \colon (X_{1},d_{1}) \to (X_{0},d_{0})$ are $1$-Lipschitz. It follows by Lemma~\ref{lemma:lipschitz.order} that $\phi \colon (X_{0},d_{0}) \to (X_{1},d_{1})$ and $\psi \colon (X_{1},d_{1}) \to (X_{0},d_{0})$ must be isometric bijections. It remains to show that $F_{0} \subseteq \overline{F_{1}} \circ \phi$ and $F_{1} \subseteq \overline{F_{0}} \circ \psi$. Thanks to symmetry, it suffices to verify that $F_{0} \subseteq \overline{F_{1}} \circ \phi$. To this end, we first show that \begin{equation}\tag{$\ast$}\label{k-approximation}
	\forall k \in \mathbb{N} \colon \quad \{ (f \wedge k) \vee (-k) \mid f \in F_{0} \} \, \subseteq \, \overline{\{ (f \wedge k) \vee (-k) \mid f \in F_{1} \circ \phi \}} .
\end{equation} Let $k \in \mathbb{N}$. Consider \begin{displaymath}
	H_{i,k} \defeq \overline{\{ (f \wedge k) \vee (-k) \mid f \in F_{i} \}} = \overline{\left\{ (f \wedge k) \vee (-k) \left\vert \, f \in \overline{F_{i}} \right\} \! \right.} \quad (i \in \{ 0,1 \}) ,
\end{displaymath} where the closure operators refer to the respective topologies of pointwise convergence. Thanks to Lemma~\ref{lemma:compact.lipschitz.spaces}, $\left(H_{0,k},\mathrm{me}_{\mu_{0}}\right)$ and $\left(H_{1,k},\mathrm{me}_{\mu_{1}}\right)$ are compact metric spaces. Moreover, we obtain well-defined isometric maps \begin{align*}
	& \Phi \colon \left(H_{1,k},\mathrm{me}_{\mu_{1}}\right) \, \longrightarrow \, \left(H_{0,k},\mathrm{me}_{\mu_{0}}\right) \! , \quad f \, \longmapsto \, f \circ \phi , \\
	& \Psi \colon \left(H_{0,k},\mathrm{me}_{\mu_{0}}\right) \, \longrightarrow \, \left(H_{1,k},\mathrm{me}_{\mu_{1}}\right) \! , \quad f \, \longmapsto \, f \circ \psi .
\end{align*} Being an isometric self-map of a compact metric space, $\Phi \circ \Psi \colon H_{0,k} \to H_{0,k}$ must be surjective. Hence, \begin{align*}
	\{ (f \wedge k) \vee (-k) \mid f \in F_{0} \} \, &\subseteq \, H_{0,k} \, = \, \Phi (\Psi(H_{0,k})) \\
	&\subseteq \, \Phi (H_{1,k}) \, = \, \overline{\{ (f \wedge k) \vee (-k) \mid f \in F_{1} \circ \phi \}} .
\end{align*} This proves~\eqref{k-approximation}. In order to deduce that $F_{0} \subseteq \overline{F_{1}} \circ \phi$, let $f \in F_{0}$. Consider any finite subset $E \subseteq X_{0}$ and $\epsilon > 0$. Let $k \defeq \sup_{x \in E} \vert f(x) \vert + 1 + \epsilon$. By~\eqref{k-approximation}, there exists $g \in F_{1} \circ \phi$ such that $\sup_{x \in E} \vert ((f(x) \wedge k) \vee (-k)) - ((g(x) \wedge k) \vee (-k)) \vert \leq \epsilon$. Since \begin{displaymath}
	f(x) \, \in \, [-k+1+\epsilon,k-1-\epsilon]
\end{displaymath} for each $x \in E$, we have $((f \wedge k) \vee (-k))\vert_{E} = f\vert_{E}$. It follows that \begin{displaymath}
	(g(x) \wedge k) \vee (-k) \, \in \, [-k+1,k+1]
\end{displaymath} for each $x \in E$, whence $((g \wedge k) \vee (-k))\vert_{E} = g\vert_{E}$. Thus, $\sup_{x \in E} \vert f(x) - g(x) \vert \leq \epsilon$. This shows that $f \in \overline{F_{1} \circ \phi} = \overline{F_{1}} \circ \phi$, as desired. \end{proof}

We now proceed to some prerequisites necessary for the proof of Theorem~\ref{theorem:observable.metric}. Our first lemma will settle the triangle inequality.

\begin{lemma}\label{lemma:reparametrization} Let $\mathscr{D} = (X,F,\mu)$ be a geometric data set and let $\phi,\psi \colon I \to X$ be any two parametrizations of $\mu$. Then, for every $\epsilon > 0$, there exist Borel isomorphisms $g,h \colon I \to I$ with $g_{\ast}(\lambda) = h_{\ast}(\lambda) = \lambda$ and $\sup_{f \in F} \Vert (f \circ \phi \circ g) - (f \circ \psi \circ h) \Vert_{\infty} \leq \epsilon$. \end{lemma}

\begin{proof} Let $\epsilon > 0$. Since $(X,d_{F})$ is separable, we find a sequence of pairwise disjoint Borel subsets $B_{n} \subseteq X$ $(n \geq 1)$ such that \begin{itemize}
	\item[$-$] $\sup_{n \geq 1} \sup_{f \in F} \diam f(B_{n}) \leq \epsilon$,
	\item[$-$] $\sum_{n=1}^{\infty} \mu (B_{n}) = 1$,
	\item[$-$] $\mu (B_{n}) > 0$ for all $n \geq 1$.
        \end{itemize}

        Let $b_{0} \defeq 0$. For each $n \geq 1$, let $a_{n} \defeq \mu
(B_{n}) = \lambda (\phi^{-1}(B_{n})) = \lambda (\psi^{-1}(B_{n}))$ and let $b_{n} \defeq \sum_{j=1}^{n} a_{j}$. Due to~\citet[(17.41)]{KechrisBook}, for each $n \geq 1$ there exists a Borel isomorphism $g_{n} \colon [b_{n-1},b_{n}) \to \phi^{-1}(B_{n})$ such that $(g_{n})_{\ast}(\lambda \!\!\upharpoonright_{[b_{n-1},b_{n})}) = \lambda \!\!\upharpoonright_{\phi^{-1}(B_{n})}$. The map $g \colon I \to I$ defined by $g\vert_{[b_{n-1},b_{n})} = g_{n}$ for all $n \geq 1$ is a Borel isomorphism with $g_{\ast}(\lambda) = \lambda$ and $g([b_{n-1},b_{n})) = \phi^{-1}(B_{n})$ for each $n \geq 1$. Similarly, we find a Borel isomorphism $h \colon I \to I$ with $h_{\ast}(\lambda) = \lambda$ and $h([b_{n-1},b_{n})) = \psi^{-1}(B_{n})$ for all $n \geq 1$. It remains to show that $\sup_{f \in F} \Vert (f \circ \phi \circ g) - (f \circ \psi \circ h) \Vert_{\infty} \leq \epsilon$. Indeed, for every $t \in I$, there exists some $n \geq 1$ with $t \in [b_{n-1},b_{n})$, whence $\{ \phi (g(t)), \psi (h(t)) \} \subseteq B_{n}$ and therefore $\sup\nolimits_{f \in F} \vert f(\phi(g(t))) - f(\psi(h(t))) \vert \leq \epsilon$. This completes the argument. \end{proof}

\begin{lemma}\label{lemma:triangle.inequality} For any three geometric data sets $\mathscr{D}_{i} = (X_{i},F_{i},\mu_{i})$ $(i \in \{ 0,1,2 \})$, \begin{displaymath}
	d_{\mathrm{conc}}(\mathscr{D}_{0},\mathscr{D}_{2}) \, \leq \, d_{\mathrm{conc}}(\mathscr{D}_{0},\mathscr{D}_{1}) + d_{\mathrm{conc}}(\mathscr{D}_{1},\mathscr{D}_{2}) .
\end{displaymath} \end{lemma}

\begin{proof} We will prove that $d_{\mathrm{conc}}(\mathscr{D}_{0},\mathscr{D}_{2}) \leq d_{\mathrm{conc}}(\mathscr{D}_{0},\mathscr{D}_{1}) + d_{\mathrm{conc}}(\mathscr{D}_{1},\mathscr{D}_{2}) + \epsilon$ for all $\epsilon > 0$. To this end, let $\epsilon > 0$ and pick parametrizations $\phi_{0}$ for $\mu_{0}$, $\phi_{1}$ and $\phi_{1}'$ for $\mu_{1}$, and $\phi_{2}$ for $\mu_{2}$ such that $(\mathrm{me}_{\lambda})_{\mathrm{H}}(F_{0} \circ \phi_{0},F_{1} \circ \phi_{1}) < d_{\mathrm{conc}}(\mathscr{D}_{0},\mathscr{D}_{1}) + \tfrac{\epsilon}{3}$ and $(\mathrm{me}_{\lambda})_{\mathrm{H}}(F_{1} \circ \phi_{1}',F_{2} \circ \phi_{2}) < d_{\mathrm{conc}}(\mathscr{D}_{1},\mathscr{D}_{2}) + \tfrac{\epsilon}{3}$. By Lemma~\ref{lemma:reparametrization}, there exist Borel isomorphisms $g,h \colon I \to I$ such that $g_{\ast}(\lambda) = h_{\ast}(\lambda) = \lambda$ and \[\sup\nolimits_{f \in F_{1}} \Vert (f \circ \phi_{1} \circ g) - (f \circ \phi_{1}' \circ h) \Vert_{\infty}\, \leq\, \tfrac{\epsilon}{3}.\] Evidently, $\phi_{0} \circ g$ is a parametrization for $\mu_{0}$, while $\phi_{2} \circ h$ is a parametrization for $\mu_{2}$. In turn, \begin{align*}
	d_{\mathrm{conc}}(\mathscr{D}_{0},\mathscr{D}_{2}) \, &\leq \, (\mathrm{me}_{\lambda})_{\mathrm{H}}(F_{0} \circ \phi_{0} \circ g, F_{2} \circ \phi_{2} \circ h) \\
	&\leq \, (\mathrm{me}_{\lambda})_{\mathrm{H}}(F_{0} \circ \phi_{0} \circ g, F_{1} \circ \phi_{1} \circ g) + (\mathrm{me}_{\lambda})_{\mathrm{H}}(F_{1} \circ \phi_{1} \circ g, F_{1} \circ \phi_{1}' \circ h) \\
	& \ \ \qquad \qquad \qquad \qquad \qquad \qquad \quad \ \! + (\mathrm{me}_{\lambda})_{\mathrm{H}}(F_{1} \circ \phi_{1}' \circ h, F_{2} \circ \phi_{2} \circ h) \\
	&\leq \, \left( d_{\mathrm{conc}}(\mathscr{D}_{0},\mathscr{D}_{1}) + \tfrac{\epsilon}{3}\right) + \tfrac{\epsilon}{3} + \left( d_{\mathrm{conc}}(\mathscr{D}_{1},\mathscr{D}_{2}) + \tfrac{\epsilon}{3}\right) \\
	&\leq \, d_{\mathrm{conc}}(\mathscr{D}_{0},\mathscr{D}_{1}) +
   d_{\mathrm{conc}}(\mathscr{D}_{1},\mathscr{D}_{2}) + \epsilon .\qedhere
\end{align*}  \end{proof}

Let us also note the following basic fact about complete metric spaces.

\begin{lemma}\label{lemma:dichotomy} Let $(X,d)$ be a complete metric space. If $(x_{n})_{n \in \mathbb{N}} \in X^{\mathbb{N}}$ and $\xi$ is an ultrafilter on $\mathbb{N}$, then either $(x_{n})_{n \in \mathbb{N}}$ converges in $(X,d)$ along $\xi$, or there exists $\epsilon > 0$ such that \begin{displaymath}
	\forall K \subseteq X \, \textit{compact } \colon \quad \{ n \in \mathbb{N} \mid K \cap B_{d}(x_{n},\epsilon) = \emptyset \} \, \in \, \xi .
      \end{displaymath} \end{lemma}
  
\begin{proof} Let $(x_{n})_{n \in \mathbb{N}} \in X^{\mathbb{N}}$ and let $\xi$ be an ultrafilter on $\mathbb{N}$. Clearly, the two alternatives are mutually exclusive: if $(x_{n})_{n \in \mathbb{N}}$ converges in $(X,d)$ along $\xi$ to some $x \in X$, then, for every $\epsilon > 0$, it follows that \begin{displaymath}
		\xi \, \ni \, \{ n \in \mathbb{N} \mid d(x_{n},x) < \epsilon \} \, = \, \{ n \in \mathbb{N} \mid \{x\} \cap B_{d}(x_{n},\epsilon) \ne \emptyset \} ,
\end{displaymath} that is, $\{ n \in \mathbb{N} \mid \{x\} \cap
B_{d}(x_{n},\epsilon) = \emptyset \} \notin \xi$. To prove the desired
conclusion, suppose that, for every $\epsilon > 0$, there exists a compact
subset $K \subseteq X$ such that \[\{ n \in \mathbb{N} \mid K \cap
  B_{d}(x_{n},\epsilon) \ne \emptyset \}\, \in\, \xi.\] Hence, for every $m \in
\mathbb{N}_{\geq 1}$, there exist a compact subset $K_{m} \subseteq X$ as well as a sequence $(x_{n}^{m})_{n \in \mathbb{N}} \in (K_{m})^{\mathbb{N}}$ such that $\left\{ n \in \mathbb{N} \left\vert \, d(x_{n}^{m},x_{n}) < \tfrac{1}{m} \right\} \in \xi \right.$. Let $x^{m} \defeq \lim_{n \to \xi} x_{n}^{m} \in K_{m}$ for all $m \in \mathbb{N}_{\geq 1}$. Since $\left\{ n \in \mathbb{N} \left\vert \, d(x^{m},x_{n}^{m}) < \tfrac{1}{m} \right\} \in \xi \right.$ and $\left\{ n \in \mathbb{N} \left\vert \, d(x_{n}^{m},x_{n}) < \tfrac{1}{m} \right\} \in \xi \right.$ for all $m \in \mathbb{N}_{\geq 1}$, it follows that \begin{equation}\tag{$\ast$}\label{filter}
	\forall m \in \mathbb{N}_{\geq 1} \colon \quad \left\{ n \in \mathbb{N} \left\vert \, d(x^{m},x_{n}) < \tfrac{2}{m} \right\} \, \in \, \xi . \right.
\end{equation} Since $\xi$ is a proper filter, \eqref{filter} readily implies that $d(x^{m},x^{\ell}) < \tfrac{4}{\min (m,\ell)}$ for any two positive integers $m,\ell \in \mathbb{N}_{\geq 1}$. Therefore, the sequence $(x^{m})_{m \geq 1}$ is Cauchy with respect to $d$. As $(X,d)$ is complete, $(x^{m})_{m \geq 1}$ thus converges to some point $x \in X$. Appealing to~\eqref{filter} again, we conclude that $x_{n} \longrightarrow x$ as $n \to \xi$, which completes the argument. \end{proof}

\begin{corollary}\label{corollary:dichotomy} Let $(X,d,\mu)$ be an $mm$-space. If $(x_{n})_{n \in \mathbb{N}} \in X^{\mathbb{N}}$ and $\xi$ is an ultrafilter on $\mathbb{N}$, then either $(x_{n})_{n \in \mathbb{N}}$ converges in $(X,d)$ along $\xi$, or there exists $\epsilon > 0$ such that \begin{displaymath}
	\lim\nolimits_{n\to \xi} \mu (B_{d}(x_{n},\epsilon)) \, = \, 0 .
\end{displaymath} \end{corollary}

\begin{proof} Let us note that the two alternatives are mutually exclusive: if $(x_{n})_{n \in \mathbb{N}}$ converges in $(X,d)$ along $\xi$ to some $x \in X$, then, for every $\epsilon > 0$, it follows that \begin{displaymath}
	\xi \, \ni \, \left\{ n \in \mathbb{N} \left\vert \, d(x_{n},x) < \tfrac{\epsilon}{2} \right\} \right. \subseteq \, \left\{ n \in \mathbb{N} \left\vert \, B_{d}\!\left(x,\tfrac{\epsilon}{2}\right) \subseteq B_{d}(x_{n},\epsilon) \right\} , \right.
\end{displaymath} whence $\lim_{n\to \xi}\mu(B_{d}(x_{n},\epsilon)) \geq \mu \!\left(B_{d}\!\left(x,\tfrac{\epsilon}{2}\right)\right) > 0$ as $\spt \mu = X$. Let us suppose now that the sequence $(x_{n})_{n \in \mathbb{N}}$ does not converge in $(X,d)$ along $\xi$. By Lemma~\ref{lemma:dichotomy}, there exists $\epsilon > 0$ such that $\{ n \in \mathbb{N} \mid K \cap B_{d}(x_{n},\epsilon) = \emptyset \} \in \xi$ for every compact subset $K \subseteq X$. We show that $\lim_{n \to \xi}\mu (B_{d}(x_{n},\epsilon)) = 0$. To this end, let $\delta > 0$. Being a Borel probability measure on a Polish space, $\mu_{i}$ must be regular \citep[e.g.,][Chapter~II, Theorem~3.2]{part}. Hence, there is a compact subset $K \subseteq X$ with $\mu (K) \geq 1- \delta$. By choice of $\epsilon$, it follows that \begin{displaymath}
	\xi \, \ni \, \{ n \in \mathbb{N} \mid K \cap B_{d}(x_{n},\epsilon) = \emptyset \} \, \subseteq \, \{ n \in \mathbb{N} \mid \mu(B_{d}(x_{n},\epsilon)) \leq \delta \} ,
\end{displaymath} thus $\lim_{n \to \xi} \mu (B_{d}(x_{n},\epsilon)) \leq \delta$ as desired. \end{proof}

Everything is in place to prove the desired theorem. Our argument resembles an idea by~\citet[Proof of Theorem~7.4.8]{PestovBook}.

\begin{theorem}\label{theorem:observable.metric} $d_{\mathrm{conc}}$ constitutes a metric on ${\mathscrbf{D}}$. \end{theorem}

\begin{proof} As observed above, $d_{\mathrm{conc}} \colon {\mathscrbf{D}} \to \mathbb{R}$ is well defined. (In fact, $d_{\mathrm{conc}}$ ranges in $[0,1]$, since $\mathrm{me}_{\lambda}$ only takes valued in $[0,1]$.) We note that $d_{\mathrm{conc}}$ is symmetric and assigns the value $0$ to identical pairs. Furthermore, $d_{\mathrm{conc}}$ satisfies the triangle inequality by Lemma~\ref{lemma:triangle.inequality}. In order to prove that $d_{\mathrm{conc}}$ separates isomorphism classes of geometric data sets, let $\mathscr{D}_{i} = (X_{i},F_{i},\mu_{i})$ $(i \in \{ 0,1 \})$ be a pair of geometric data sets such that $d_{\mathrm{conc}}(\mathscr{D}_{0},\mathscr{D}_{1}) = 0$. We wish to verify that $\mathscr{D}_{0} \cong \mathscr{D}_{1}$. Thanks to Proposition~\ref{proposition:feature.order}, it suffices to show that $\mathscr{D}_{1} \preceq \mathscr{D}_{0}$, as we will do.
	
Being Borel probability measures on Polish spaces, both $\mu_{0}$ and $\mu_{1}$ are necessarily regular (see, e.g.,~\citep[Chapter~II, Theorem~3.2]{part}). Hence, for each $n \in \mathbb{N}$ and $i \in \{ 0,1\}$, there is a compact subset $K_{i,n} \subseteq X_{i}$ with $\mu_{i}(K_{i,n}) \geq 1-2^{-n}$. A straightforward compactness argument now reveals that, for every $n \in \mathbb{N}$ and $i \in \{ 0,1\}$, there is a finite subset $F_{i,n} \subseteq F_{i}$ such that \begin{displaymath}
	\forall x,y \in K_{i,n} \colon \quad \left\lvert d_{F_{i}}(x,y) - d_{F_{i,n}}(x,y) \right\rvert \, \leq \, 2^{-n} .
\end{displaymath} For the rest of the proof, let $\phi \colon I \to X_{0}$ be a (fixed) parametrization for $\mu_{0}$.

Consider any $n \in \mathbb{N}$. Since $d_{\mathrm{conc}}(\mathscr{D}_{0},\mathscr{D}_{1}) = 0$, we find a parametrization $\phi_{n} \colon I \to X_{0}$ for $\mu_{0}$ and a parametrization $\psi'_{n} \colon I \to X_{1}$ for $\mu_{1}$ such that \begin{displaymath}
	(\mathrm{me}_{\lambda})_{\mathrm{H}}(F_{0} \circ \phi_{n},F_{1} \circ \psi'_{n} ) \, < \, \tfrac{2^{-(n+1)}}{\vert F_{0,n} \vert + \vert F_{1,n} \vert + 1} \, .
\end{displaymath} By Lemma~\ref{lemma:reparametrization}, there exists Borel isomorphisms $g,h \colon I \to I$ with $g_{\ast}(\lambda) = h_{\ast}(\lambda) = \lambda$ and $\sup_{f \in F_{0}} \Vert (f \circ \phi \circ g) - (f \circ \phi_{n} \circ h) \Vert_{\infty} < \tfrac{2^{-(n+1)}}{\vert F_{0,n} \vert + \vert F_{1,n} \vert + 1}$. It follows that $\psi_{n} \defeq \psi_{n}' \circ h \circ g^{-1} \colon I \to X_{1}$ is a parametrization for $\mu_{1}$ and, moreover, \begin{align*}
	&(\mathrm{me}_{\lambda})_{\mathrm{H}}(F_{0} \circ \phi,F_{1} \circ \psi_{n} ) \, = \, (\mathrm{me}_{\lambda})_{\mathrm{H}}(F_{0} \circ \phi,F_{1} \circ \psi_{n}' \circ h \circ g^{-1} ) \\
	& \leq \, (\mathrm{me}_{\lambda})_{\mathrm{H}}(F_{0} \circ \phi,F_{0} \circ \phi_{n} \circ h \circ g^{-1} ) + (\mathrm{me}_{\lambda})_{\mathrm{H}}(F_{0} \circ \phi_{n} \circ h \circ g^{-1},F_{1} \circ \psi_{n}' \circ h \circ g^{-1} ) \\
	& \leq \, \sup\nolimits_{f \in F_{0}} \Vert (f \circ \phi) - (f \circ \phi_{n} \circ h \circ g^{-1}) \Vert_{\infty} + (\mathrm{me}_{\lambda})_{\mathrm{H}}(F_{0} \circ \phi_{n} \circ h \circ g^{-1},F_{1} \circ \psi_{n}' \circ h \circ g^{-1} ) \\
	& = \, \sup\nolimits_{f \in F_{0}} \Vert (f \circ \phi \circ g) - (f \circ \phi_{n} \circ h) \Vert_{\infty} + (\mathrm{me}_{\lambda})_{\mathrm{H}}(F_{0} \circ \phi_{n},F_{1} \circ \psi_{n}' ) \\
	& < \, \tfrac{2^{-n}}{\vert F_{0,n} \vert + \vert F_{1,n} \vert + 1} \, .
\end{align*} In particular, for each $f \in F_{0,n}$ there exist $h_{0,n,f} \in F_{1}$ and a Borel subset $B_{0,n,f} \subseteq I$ such that \begin{displaymath}
	\lambda\!\left(B_{0,n,f}\right) \, \geq \, 1 - \tfrac{2^{-n}}{\vert F_{0,n} \vert + \vert F_{1,n} \vert + 1}, \qquad \sup\nolimits_{t \in B_{0,n,f}}\left\vert f(\phi(t)) - h_{0,n,f}(\psi_{n}(t)) \right\vert \, \leq \, 2^{-n} ,
\end{displaymath} and for each $f' \in F_{1,n}$ there exist $h_{1,n,f'} \in F_{0}$ and a Borel subset $B_{1,n,f'} \subseteq I$ such that \begin{displaymath}
	\lambda\!\left(B_{1,n,f'}\right) \geq 1 - \tfrac{2^{-n}}{\vert F_{0,n} \vert + \vert F_{1,n} \vert + 1}, \qquad \sup\nolimits_{t \in B_{1,n,f'}}\left\vert h_{1,n,f'}(\phi(t)) - f' (\psi_{n}(t)) \right\vert \, \leq \, 2^{-n} .
\end{displaymath} Let us consider the Borel subsets \begin{displaymath}
	B_{n} \defeq \bigcap\nolimits_{f \in F_{0,n}} B_{0,n,f} \cap \bigcap\nolimits_{f' \in F_{1,n}} B_{1,n,f'} , \qquad T_{n} \defeq B_{n} \cap \phi^{-1}(K_{0,n}) \cap \psi_{n}^{-1}(K_{1,n})
\end{displaymath} of $I$. Note that $\lambda (B_{n}) \geq 1 - 2^{-n}$ and thus $\lambda (T_{n}) \geq 1 - 3\cdot 2^{-n} \geq 1 - 2^{2-n}$. Moreover, \begin{displaymath}
	\sup\nolimits_{t \in B_{n}}\!\left\lvert f(\phi(t)) - h_{0,n,f}(\psi_{n}(t)) \right\rvert \! \, \leq \, 2^{-n}
\end{displaymath} for $f \in F_{0,n}$ and $\sup\nolimits_{t \in B_{n}}\!\left\lvert h_{1,n,f'}(\phi(t)) - f'(\psi_{n}(t)) \right\rvert \leq 2^{-n}$ for $f' \in F_{1,n}$. We claim that \begin{equation}\tag{$\ast$}\label{approximation}
	\forall s,t \in T_{n} \colon \quad \left\lvert d_{F_{0}}(\phi(s),\phi(t)) - d_{F_{1}}(\psi_{n}(s),\psi_{n}(t)) \right\rvert \, < \, 2^{2-n} .
\end{equation} To prove this, let $s,t \in T_{n}$. Since $\{ s,t\} \subseteq B_{n}$, it follows that \begin{align*}
	d_{F_{0,n}}(\phi(s),\phi(t)) \, &= \, \sup\nolimits_{f \in F_{0,n}} \vert f(\phi(s)) - f(\phi(t)) \vert \\
	& \leq \, \sup\nolimits_{f \in F_{0,n}} \left\lvert h_{0,n,f}(\psi_{n}(s)) - h_{0,n,f}(\psi_{n}(t)) \right\rvert + 2^{1-n} \\
	& \leq \, d_{F_{1}}(\psi_{n}(s),\psi_{n}(t)) + 2^{1-n} .
\end{align*} Also, $\vert d_{F_{0}}(\phi(s),\phi(t)) - d_{F_{0,n}}(\phi(s),\phi(t)) \vert \leq 2^{-n}$ as $\{\phi (s), \phi(t) \} \subseteq K_{0,n}$. Thus, \begin{align*}
	d_{F_{0}}&(\phi(s),\phi(t)) - d_{F_{1}}(\psi_{n}(s),\psi_{n}(t)) \, \\
	& = \, d_{F_{0}}(\phi(s),\phi(t)) - d_{F_{0,n}}(\phi(s),\phi(t)) + d_{F_{0,n}}(\phi(s),\phi(t)) - d_{F_{1}}(\psi_{n}(s),\psi_{n}(t)) \\
	& \leq \, 2^{-n} + 2^{1-n} \, = \, 3\cdot 2^{-n} \, < \, 2^{2-n} .   
\end{align*} Similarly, we observe that \begin{align*}
	d_{F_{1,n}}(\psi_{n}(s),\psi_{n}(t)) \, &= \, \sup\nolimits_{f' \in F_{1,n}} \vert f'(\psi_{n}(s)) - f'(\psi_{n}(t)) \vert \\
		& \leq \, \sup\nolimits_{f' \in F_{1,n}} \left\lvert h_{1,n,f'}(\phi(s)) - h_{1,n,f'}(\phi(t)) \right\rvert + 2^{1-n} \\
		& \leq \, d_{F_{0}}(\phi(s),\phi(t)) + 2^{1-n} ,
\end{align*} as $\{ s,t\} \subseteq B_{n}$. Furthermore, note that $\left\lvert d_{F_{1}}(\psi_{n}(s),\psi_{n}(t)) - d_{F_{1,n}}(\psi_{n}(s),\psi_{n}(t)) \right\rvert \leq 2^{-n}$, since $\{\psi_{n} (s), \psi_{n}(t) \} \subseteq K_{1,n}$. Accordingly, \begin{align*}
	d_{F_{1}}&(\psi_{n}(s),\psi_{n}(t)) - d_{F_{0}}(\phi(s),\phi(t)) \, \\
		& = \, d_{F_{1}}(\psi_{n}(s),\psi_{n}(t)) - d_{F_{1,n}}(\psi_{n}(s),\psi_{n}(t)) + d_{F_{1,n}}(\psi_{n}(s),\psi_{n}(t)) - d_{F_{0}}(\phi(s),\phi(t)) \\
		& \leq \, 2^{-n} + 2^{1-n} \, = \, 3\cdot 2^{-n} \, < \, 2^{2-n} .   
\end{align*} This proves \eqref{approximation}.

Consider the Borel subset $T \defeq \bigcup_{m \in \mathbb{N}} \bigcap_{n \geq m} T_{n} \subseteq I$. Since $\sum_{n \in \mathbb{N}} \lambda (I\setminus T_{n}) < \infty$, the Borel-Cantelli lemma asserts that $\lambda (T) = 1$. We claim that \begin{equation}\tag{$\ast \ast$}\label{liminf}
	\forall t \in T \ \forall \epsilon > 0 \colon \quad \liminf\nolimits_{n \to \infty} \mu_{1}\!\left(B_{d_{F_{1}}}(\psi_{n}(t),\epsilon)\right) \, \geq \, \mu_{0}\!\left(B_{d_{F_{0}}}(\phi(t),\epsilon)\right) .
\end{equation} To see this, let $t \in T$ and $\epsilon > 0$. Consider any $\delta > 0$. Let $m_{0} \in \mathbb{N}$ such that $t \in \bigcap_{n \geq m_{0}} T_{n}$ and $2^{2-m_{0}} < \tfrac{\delta}{2}$. Since $\mu_{0}$ is $\sigma$-additive, there exists $m \in \mathbb{N}_{\geq m_{0}}$ such that \begin{displaymath}
	\mu_{0}\!\left(B_{d_{F_{0}}}\!\left(\phi(t),\epsilon-2^{2-m}\right)\right) \, \geq \, \mu_{0}\!\left(B_{d_{F_{0}}}(\phi(t),\epsilon)\right) - \tfrac{\delta}{2} .
\end{displaymath} Also, \eqref{approximation} implies that $T_{n} \cap \phi^{-1}\!\left(B_{d_{F_{0}}}\!\left(\phi(t),\epsilon - 2^{2-n}\right)\right) \subseteq \psi_{n}^{-1}\!\left(B_{d_{F_{1}}}(\psi_{n}(t),\epsilon)\right)$ for all $n \in \mathbb{N}$. Hence, if $n \in \mathbb{N}_{\geq m}$, then \begin{align*}
	\mu_{1}\!\left(B_{d_{F_{1}}}(\psi_{n}(t),\epsilon)\right) \, & = \, \lambda \!\left(\psi_{n}^{-1}(B_{d_{F_{1}}}(\psi_{n}(t),\epsilon ))\right) \\
	& \geq \, \lambda \!\left(T_{n} \cap \phi^{-1}\!\left(B_{d_{F_{0}}}\left(\phi(t),\epsilon - 2^{2-n}\right) \right) \right) \\
	& \geq \, 1 - \lambda (I\setminus T_{n}) - \lambda \!\left(I \setminus \phi^{-1}\!\left(B_{d_{F_{0}}}\!\left(\phi(t),\epsilon - 2^{2-n}\right) \right) \right) \\
	& = \, \lambda (T_{n}) - 1 + \mu_{0} \!\left(B_{d_{F_{0}}}\!\left(\phi(t),\epsilon - 2^{2-n} \right) \right) \\
	& \geq \, - 2^{2-n} + \mu_{0} \!\left(B_{d_{F_{0}}}(\phi(t),\epsilon)\right) - \tfrac{\delta}{2} \\
	& \geq \, \mu_{0}\left(B_{d_{F_{0}}}(\phi(t),\epsilon)\right) - \delta .
\end{align*} This proves~\eqref{liminf}.

Henceforth, let $\xi$ be a (fixed) non-principal ultrafilter on
$\mathbb{N}$. Due to~\eqref{liminf} and Corollary~\ref{corollary:dichotomy}, we
may define the map $\psi \colon T \to X_{1}, \, t \mapsto \lim_{n\to \xi} \psi_{n}(t)$. By $\xi$ being non-principal, \eqref{approximation} implies that \begin{displaymath}
	\forall s,t \in T \colon \quad d_{F_{0}}(\phi(s),\phi(t)) \, = \, d_{F_{1}}(\psi(s),\psi(t)) .
\end{displaymath} So, there is a unique map $\sigma \colon \phi(T) \to X_{1}$ such that $\sigma (\phi (t)) = \psi (t)$ for all $t \in T$. Evidently, $\phi (T)$ is dense in $X_{0}$: if $U$ is a non-empty open subset of $X_{0}$, then, as $\lambda (T) = 1$ and $\spt \mu_{0} = X_{0}$, it follows that $\lambda (T \cap \phi^{-1}(U)) = \lambda (\phi^{-1}(U)) = \mu_{0}(U) > 0$, thus $\phi(T) \cap U \ne \emptyset$. Since $\sigma \colon (\phi(T),d_{F_{0}}) \to (X_{1},d_{F_{1}})$ is isometric and $(X_{1},d_{F_{1}})$ is a complete metric space, this implies the existence of a unique isometric mapping $\bar{\sigma} \colon (X_{0},d_{F_{0}}) \to (X_{1},d_{F_{1}})$ such that $\bar{\sigma}\vert_{\phi (T)} = \sigma$, i.e., $(\bar{\sigma} \circ \phi)\vert_{T} = \psi$. In particular, $\bar{\sigma}$ is Borel measurable. We will show that \begin{equation}\tag{$\ast \ast \ast$}\label{integrals}
	\forall f \in \mathrm{Lip}_{1}^{1}\!\left(X_{1},d_{F_{1}}\right) \colon \quad \int f \, \mathrm{d}\mu_{1} \, = \, \int f \circ \bar{\sigma} \, \mathrm{d}\mu_{0} .
\end{equation} Let $f \in \mathrm{Lip}_{1}^{1}(X_{1},d_{F_{1}})$ and $\epsilon >
0$. Put $\tau \defeq \tfrac{\epsilon}{6}$. Since $1 = \lambda (T) = \sup_{m \in
  \mathbb{N}} \lambda (\bigcap_{n \geq m} T_{n})$, there exists $m \in
\mathbb{N}$ such that $\lambda (\bigcap_{n \geq m} T_{n}) \geq 1 - \tau$ and $2^{2-m} \leq \tau$. Consider the Borel set $T^{\ast}_{m} \defeq \bigcap_{n \geq m} T_{n} \subseteq I$. Since $\phi(T^{\ast}_{m})$ is contained in $K_{0,m}$ and thus $d_{F_{0}}$-precompact, there exists a finite subset $E \subseteq T^{\ast}_{m}$ such that $\phi(T^{\ast}_{m}) \subseteq \bigcup_{s \in E} B_{d_{F_{0}}}(\phi(s),\tau)$. By definition of $\psi$ and non-principality of $\xi$, \begin{displaymath}
	M \, \defeq \, \! \left\{ n \in \mathbb{N}_{\geq m} \left\vert \, \forall s \in E \colon \, d_{F_{1}}(\psi_{n}(s),\psi(s)) < \tau \right\} \! \, \in \, \xi . \right.
\end{displaymath} In particular, $M \ne \emptyset$. Pick any $n \in M$. Then $\sup_{t \in T^{\ast}_{m}} \vert f(\psi_{n}(t)) - f(\psi(t)) \vert \leq 4\tau$. Indeed, if $t \in T^{\ast}_{m}$, then there exists $s \in E$ such that $d_{F_{0}}(\phi(s),\phi(t)) < \tau$, whence \begin{align*}
	\vert f(\psi_{n}(t)) - f(\psi (t)) \vert \, & \leq \, d_{F_{1}}(\psi_{n}(t),\psi(t)) \\
		& \leq \, d_{F_{1}}(\psi_{n}(t),\psi_{n}(s)) + d_{F_{1}}(\psi_{n}(s),\psi(s)) + d_{F_{1}}(\psi(s),\psi(t)) \\
		& \leq \, d_{F_{0}}(\phi(t),\phi(s)) + 2^{2-n} + \tau + d_{F_{0}}(\phi(s),\phi(t)) \, \leq \, 4\tau 
\end{align*} by~\eqref{approximation}. We conclude that \begin{align*}
	\left\lvert \int f \circ \bar{\sigma} \, \mathrm{d}\mu_{0} - \int f \, \mathrm{d}\mu_{1} \right\rvert \, & = \, \left\lvert \int f \circ \bar{\sigma} \circ \phi \, \mathrm{d}\lambda - \int f \circ \psi_{n} \, \mathrm{d}\lambda \right\rvert \\
		& \leq \, \int_{T^{\ast}_{m}} \vert f(\psi(t)) - f(\psi_{n}(t)) \vert \, \mathrm{d}\lambda (t) + 2\lambda(I\setminus T^{\ast}_{m}) \\
		& \leq \, 4\tau + 2\tau \, = \, \epsilon ,
\end{align*} proving~\eqref{integrals}. As
$\mathrm{Lip}_{1}^{1}(X_{1},d_{F_{1}})$ spans a $\Vert \cdot
\Vert_{\infty}$-dense linear subspace of the Banach space of uniformly
continuous bounded real-valued functions on $X_{1}$ \citep[Lemma
5.20(2)]{pachl}, assertion~\eqref{integrals} implies that $\int f \,
\mathrm{d}\mu_{1} =  \int f \circ \bar{\sigma} \, \mathrm{d}\mu_{0}$ for
every uniformly continuous bounded function $f\colon X_{1}\to
\mathbb{R}$. Since both $\bar{\sigma}_{\ast}(\mu_{0})$ and $\mu_{1}$ are regular
Borel probability measures on $X_{1}$, it follows that $\bar{\sigma}_{\ast}(\mu_{0}) = \mu_{1}$ \citep[Theorem 5.3]{pachl}.

It only remains to verify that $F_{1} \circ \bar{\sigma} \subseteq \overline{F_{0}}$. For this, let $f \in F_{1}$. For each $n \in \mathbb{N}$, since $(\mathrm{me}_{\lambda})_{\mathrm{H}}(F_{0} \circ \phi,F_{1} \circ \psi_{n} ) \, < \, 2^{-n}$, we find some $f_{n} \in F_{0}$ as well as a Borel subset $Q_{n} \subseteq I$ such that $\sup_{t \in Q_{n}} \vert f_{n}(\phi(t)) - f(\psi_{n}(t)) \vert \leq 2^{-n}$ and $\lambda (Q_{n}) \geq 1 - 2^{-n}$. Since $\sum_{n \in \mathbb{N}} \lambda (I\setminus Q_{n}) < \infty$, the Borel-Cantelli lemma ensures that $\lambda (Q) = 1$ for the Borel set $Q \defeq \bigcup_{m \in \mathbb{N}} \bigcap_{n \geq m} Q_{n} \subseteq I$. Consequently, $\lambda (T \cap Q) = 1$. It follows that $\phi (T \cap Q)$ is dense in $X_{0}$: again, if $U$ is a non-empty open subset of $X_{0}$, then \begin{displaymath}
	\lambda (T \cap Q \cap \phi^{-1}(U)) \, = \, \lambda (\phi^{-1}(U)) \, = \, \mu_{0}(U) \, > \, 0
\end{displaymath} as $\spt \mu_{0} = X_{0}$, and therefore $\phi(T \cap Q) \cap U \ne \emptyset$. Furthermore, by definition of $\psi$ and non-principality of $\xi$, our choice of $(f_{n})_{n \in \mathbb{N}}$ and $(Q_{n})_{n \in \mathbb{N}}$ entails that \begin{displaymath}
	\forall t \in T \cap Q \colon \qquad f_{n}(\phi(t)) \longrightarrow f(\psi(t)) \quad (n \to \xi) .
\end{displaymath} It readily follows that \begin{displaymath}
	\forall x \in X_{0} \colon \qquad f_{n}(x) \longrightarrow f(\bar{\sigma}(x)) \quad (n \to \xi) .
\end{displaymath} Indeed, if $x \in X_{0}$ and $\epsilon > 0$, then density of $\phi(T \cap Q)$ in $X_{0}$ implies the existence of $t \in T \cap Q$ with $d_{F_{0}}(x,\phi(t)) < \tfrac{\epsilon}{3}$, and so \begin{displaymath}
	\vert f(\psi(t)) - f(\bar{\sigma}(x)) \vert \, \leq \, d_{F_{1}}(\psi(t),\bar{\sigma}(x)) \, = \, d_{F_{1}}(\bar{\sigma}(\phi(t)),\bar{\sigma}(x)) \, = \, d_{F_{0}}(\phi(t),x) \leq \tfrac{\epsilon}{3} ,
\end{displaymath} thus \begin{align*}
	\vert f_{n}(x) - f(\bar{\sigma}(x)) \vert \, &\leq \vert f_{n}(x) - f_{n}(\phi(t)) \vert + \vert f_{n}(\phi(t)) - f(\psi(t)) \vert + \vert f(\psi(t)) - f(\bar{\sigma}(x)) \vert \\
	& \leq \, \epsilon
\end{align*} for all $n \in \left\{ m \in \mathbb{N} \left\vert \, \vert f_{m}(\phi(t)) - f(\psi(t)) \vert < \tfrac{\epsilon}{3} \right\} \in \xi \right.$. Hence, $f \circ \bar{\sigma} \in \overline{F_{0}}$ as desired. This shows that $\mathscr{D}_{1} \preceq \mathscr{D}_{0}$, which completes the proof. \end{proof}

The metric $d_{\mathrm{conc}}$ induces a topology on ${\mathscrbf{D}}$, the \emph{concentration topology}. The authors do not know whether the metric space $({\mathscrbf{D}},d_{\mathrm{conc}})$ is separable.

\begin{definition}[concentration of data] A sequence of geometric data sets $(\mathscr{D}_{n})_{n \in \mathbb{N}}$ is said to \emph{concentrate to} a geometric data set $\mathscr{D}$ if $d_{\mathrm{conc}}(\mathscr{D}_{n},\mathscr{D}) \longrightarrow 0$ as $n \to \infty$. \end{definition}

The concentration topology is a conceptual extension of the phenomenon of measure concentration. We refer to the latter as the \emph{L\'evy property}.

\begin{definition} A sequence of geometric data sets $\mathscr{D}_{n} = (X_{n},F_{n},\mu_{n})$ $(n \in \mathbb{N})$ is said to have the \emph{L\'evy property} or to be a \emph{L\'evy family}, resp., if \begin{displaymath}
	\sup\nolimits_{f \in F_{n}} \inf\nolimits_{c \in \mathbb{R}} \mathrm{me}_{\mu_{n}}(f,c) \, \longrightarrow \, 0 \qquad (n \to \infty) .
\end{displaymath} \end{definition}

The subsequent proposition, which is completely analogous to the corresponding result for $mm$-spaces~\citep[Lemma~5.6]{ShioyaBook}, describes the connection between the Lévy property and observable distance.

\begin{proposition}\label{proposition:levy.vs.concentration} For every geometric data set $\mathscr{D} = (X,F,\mu)$, \begin{displaymath}
    d_{\mathrm{conc}}( \mathscr{D}_{\#},\bot) \, = \, \sup\nolimits_{f \in F} \inf\nolimits_{c \in \mathbb{R}} \mathrm{me}_{\mu}(f,c)
\end{displaymath} where $\mathscr{D}_{\#} \defeq (X,F \cup \mathbb{R},\mu)$ and $\bot \defeq (\{ \emptyset \},\mathbb{R},\nu_{\{ \emptyset \}})$. In particular, a sequence of geometric data sets $(\mathscr{D}_{n})_{n \in \mathbb{N}}$ has the L\'evy property if and only if $((\mathscr{D}_{n})_{\#})_{n \in \mathbb{N}}$ concentrates to the (trivial) geometric data set $\bot$. \end{proposition}



\section{Observable Diameters of Data}\label{sec:obsdiam} We are going to adapt
Gromov's concept of \emph{observable diameter} \citep[Chapter~3$\tfrac{1}{2}$]{Gromov99} to our setup of data sets and study its behavior with respect to the concentration topology. This is a necessary preparatory step towards Section~\ref{section:intrinsic.dimension}.

\begin{definition}[observable diameter]\label{definition:observable.diameter} Let $\alpha \geq 0$. The \emph{$\alpha$-partial diameter} of a Borel probability measure $\nu$ on $\mathbb{R}$ is defined as \begin{displaymath}
    \mathrm{PartDiam}(\nu,1-\alpha) \defeq \inf \{ \diam (B) \mid B \subseteq \mathbb{R} \text{ Borel, } \nu (B) \geq 1-\alpha \} \, \in \, [0,\infty] .
\end{displaymath} We define the \emph{$\alpha$-observable diameter} of a geometric data set $\mathscr{D} = (X,F,\mu)$ to be \begin{displaymath}
    \mathrm{ObsDiam}(\mathscr{D};-\alpha) \defeq \sup \{ \mathrm{PartDiam}(f_{\ast}(\mu),1-\alpha) \mid f \in F \} \, \in \, [0,\infty] .
\end{displaymath} \end{definition}

\begin{remark}\label{remark:partial.diameter} Let $\nu$ be a Borel probability measure on $\mathbb{R}$ and let $\alpha > 0$. For any $x \in X$ there exists $n \in \mathbb{N}_{\geq 1}$ with $\nu \!\left(B_{d_{\mathbb{R}}}(x,n)\right) \geq 1-\alpha$, which readily implies that \begin{displaymath}
	\mathrm{PartDiam}(\nu,1-\alpha) \, \leq \, 2n .
\end{displaymath} In particular, $\mathrm{PartDiam}(\nu,1-\alpha) < \infty$. \end{remark}

As is easily seen, observable diameters are invariant under isomorphisms of geometric data sets,
which means that  $\mathrm{ObsDiam}(\mathscr{D}_{0};-\alpha) =
\mathrm{ObsDiam}(\mathscr{D}_{1};-\alpha)$ for any pair of isomorphic geometric
data sets $\mathscr{D}_{0} \cong \mathscr{D}_{1}$ and $\alpha \geq
0$. Furthermore, we have the following continuity with respect to  $d_{\mathrm{conc}}$.

\begin{lemma}\label{lemma:estimate} Let $\delta \defeq d_{\mathrm{conc}}(\mathscr{D}_{0},\mathscr{D}_{1})$ for geometric data sets $\mathscr{D}_{i} = (X_{i},F_{i},\mu_{i})$ $(i \in \{ 0,1 \})$. For all $\tau > \delta$ and $\alpha > 0$, \begin{displaymath}
	\mathrm{ObsDiam}(\mathscr{D}_{1};-(\alpha + \tau)) \, \leq \, \mathrm{ObsDiam}(\mathscr{D}_{0};-\alpha) + 2\tau .
\end{displaymath} \end{lemma}

\begin{proof} Let $\alpha > 0$. It suffices check that \begin{displaymath}
	\forall \kappa > 1 \colon \quad \mathrm{ObsDiam}(\mathscr{D}_{1};-(\alpha + \tau)) \, \leq \, (\mathrm{ObsDiam}(\mathscr{D}_{0};-\alpha) + 2\tau ) \cdot \kappa  .
\end{displaymath} Let $\kappa > 1$. Choose parametrizations, $\phi_{0}$ for $\mu_{0}$ and $\phi_{1}$ for $\mu_{1}$, such that \begin{displaymath}
	(\mathrm{me}_{\lambda})_{\mathrm{H}}(F_{0} \circ \phi_{0},F_{1} \circ \phi_{1}) \, < \, \tau .
\end{displaymath} Let $f_{1} \in F_{1}$. Then there is some $f_{0} \in F_{0}$ such that $\mathrm{me}_{\lambda}(f_{0}\circ \phi_{0},f_{1} \circ \phi_{1}) < \tau$. Fix any Borel subset $B \subseteq \mathbb{R}$ with $\diam (B) \leq \mathrm{ObsDiam}(\mathscr{D}_{0};-\alpha) \cdot \kappa$ and $(f_{0})_{\ast}(\mu_{0})(B) \geq 1-\alpha$. Considering the open subset $C \defeq B_{d_{\mathbb{R}}}(B,\tau \kappa) \subseteq \mathbb{R}$, we note that \begin{align*}
  (f_{1})_{\ast}(&\mu_{1})(C) \, = \, (f_{1} \circ \phi_{1})_{\ast}(\lambda)(C) \, = \, \lambda ((f_{1} \circ \phi_{1})^{-1} (C)) \\
                 &\geq \, \lambda ((f_{0} \circ \phi_{0})^{-1} (B)) - \tau \, = \, (f_{0} \circ \phi_{0})_{\ast}(\lambda)(B) - \tau \, = \, (f_{0})_{\ast}(\mu_{0})(B) - \tau \\
                 &\geq \, 1 - \alpha - \tau \, = \, 1 - (\alpha +\tau )
\end{align*} and $\diam (C) \leq \diam (B) + 2 \tau \kappa \leq (\mathrm{ObsDiam}(\mathscr{D}_{0};-\alpha) + 2\tau )\kappa$, which proves that \begin{displaymath}
	\mathrm{PartDiam}((f_{1})_{\ast}(\mu_{1}),1-(\alpha + \tau)) \leq (\mathrm{ObsDiam}(\mathscr{D}_{0};-\alpha) + 2\tau ) \kappa . \qedhere
\end{displaymath} \end{proof}

In Proposition~\ref{proposition:observable.diameters.continuous} below, we introduce a
quantity for geometric data sets, which is well defined by the following fact.

\begin{remark}\label{remark:antitone} If $\mathscr{D}$ is any geometric data set, then $[0,\infty) \to [0,\infty], \, \alpha \mapsto \mathrm{ObsDiam}(\mathscr{D};-\alpha)$ is antitone, thus Borel measurable. \end{remark}


\begin{proposition}\label{proposition:observable.diameters.continuous} The map $\Delta \colon {\mathscrbf{D}} \to [0,1]$ defined by \begin{displaymath}
	\Delta (\mathscr{D}) \defeq \int_{0}^{1} \mathrm{ObsDiam}(\mathscr{D};-\alpha) \wedge 1 \, \mathrm{d}\alpha \qquad (\mathscr{D} \in {\mathscrbf{D}})
\end{displaymath} is Lipschitz with respect to $d_{\mathrm{conc}}$. \end{proposition}

\begin{proof} Let $\delta \defeq d_{\mathrm{conc}}(\mathscr{D}_{0},\mathscr{D}_{1})$ for geometric data sets $\mathscr{D}_{i} = (X_{i},F_{i},\mu_{i})$ $(i \in \{ 0,1 \})$. Without loss of generality, we assume that $\delta < 1$. For every $\tau \in (\delta,1)$, \begin{align*}
	\Delta (\mathscr{D}_{1}) \, & \leq \, \tau + \int_{\tau}^{1} \mathrm{ObsDiam}(\mathscr{D}_{1};-\alpha) \wedge 1 \, \mathrm{d}\alpha 
	\\
	& = \, \tau + \int_{0}^{1-\tau} \mathrm{ObsDiam}(\mathscr{D}_{1};-(\alpha + \tau)) \wedge 1 \, \mathrm{d}\alpha \\
	&\leq \, 3 \tau + \int_{0}^{1-\tau} \mathrm{ObsDiam}(\mathscr{D}_{0};-\alpha) \wedge 1 \, \mathrm{d}\alpha \, \leq \, 3\tau + \Delta (\mathscr{D}_{0})
\end{align*} due to Lemma~\ref{lemma:estimate}. Hence, $\Delta (\mathscr{D}_{1}) \leq \Delta (\mathscr{D}_{0}) + 3\delta$. Thanks to symmetry, it readily follows that $\vert \Delta (\mathscr{D}_{0}) - \Delta (\mathscr{D}_{1}) \vert \leq 3\delta$, i.e., $\Delta$ is $3$-Lipschitz with respect to $d_{\mathrm{conc}}$. \end{proof}

Observable diameters reflect the Lévy property in a natural manner.

\begin{proposition} \label{proposition:observable.diameters.levy} Let $\mathscr{D}_{n} = (X_{n},F_{n},\mu_{n})$ $(n \in \mathbb{N})$ be a sequence of geometric data sets. Then the following are equivalent. \begin{description}
  \item[$(1)\,$] $(\mathscr{D}_{n})_{n \in \mathbb{N}}$ has the L\'evy property.
  \item[$(2)\,$] $\lim_{n \to \infty} \mathrm{ObsDiam}(\mathscr{D}_{n};-\alpha) = 0$ for every $\alpha > 0$.
  \item[$(3)\,$] $\lim_{n \to \infty} \Delta (\mathscr{D}_{n}) = 0$.
\end{description} \end{proposition}

\begin{proof} (1)$\Longrightarrow$(2). Let $\alpha > 0$. To prove that $\mathrm{ObsDiam}(\mathscr{D}_{n};-\alpha) \longrightarrow 0$ as $n \to \infty$, let $\epsilon > 0$. By assumption, there exists $m \in \mathbb{N}$ such that \begin{displaymath}
	\forall n \in \mathbb{N}_{\geq m} \colon \qquad \sup\nolimits_{f \in F_{n}} \inf\nolimits_{c \in \mathbb{R}} \mathrm{me}_{\mu_{n}}(f,c) \, < \, \min \left\{ \tfrac{\epsilon}{4},\alpha \right\} .
\end{displaymath} We show that $\mathrm{ObsDiam}(\mathscr{D}_{n};-\alpha) \leq \epsilon$ for all $n \in \mathbb{N}_{n\geq m}$. Let $n \in \mathbb{N}_{\geq m}$. For every $f \in F_{n}$, there exists $c \in \mathbb{R}$ with $\mathrm{me}_{\mu_{n}}(f,c) < \min \left\{ \tfrac{\epsilon}{4},\alpha \right\}$, whence \begin{displaymath}
	f_{\ast}(\mu_{n})(B) \, = \, \mu_{n}\!\left(f^{-1}(B)\right) \, \geq \, 1-\alpha
\end{displaymath} for the Borel set $B \defeq B_{d_{\mathbb{R}}}\! \left(c,\tfrac{\epsilon}{2}\right) \subseteq \mathbb{R}$. Also, $\diam (B) \leq \epsilon$. Therefore, \begin{displaymath}
	\mathrm{PartDiam}(f_{\ast}(\mu_{n}),1-\alpha) \, \leq \, \epsilon
\end{displaymath} for all $f \in F_{n}$, that is, $\mathrm{ObsDiam}(\mathscr{D}_{n};-\alpha) \leq \epsilon$.

(2)$\Longrightarrow$(1). Let $\epsilon \in (0,1)$. By our hypothesis, there exists some
$m \in \mathbb{N}$ such that $\mathrm{ObsDiam}(\mathscr{D}_{n};-\epsilon) \leq
\epsilon$ for all $n \in \mathbb{N}_{\geq m}$. We will show that \begin{displaymath}
	\forall n \in \mathbb{N}_{\geq m} \colon \qquad \sup\nolimits_{f \in F_{n}} \inf\nolimits_{c \in \mathbb{R}} \mathrm{me}_{\mu_{n}}(f,c) \, \leq \, \epsilon .
\end{displaymath} Let $n \in \mathbb{N}_{\geq m}$. For any $f \in F_{n}$ and $\delta > 0$, we find some (necessarily non-empty) Borel subset $B \subseteq \mathbb{R}$ with $f_{\ast}(\mu_{n})(B) \geq 1-\epsilon$ and $\diam (B) \leq \epsilon + \delta$, and observe that\linebreak $\mathrm{me}_{\mu_{n}}(f,c) \leq \epsilon + \delta$ for any $c \in B$. Thus, $\sup\nolimits_{f \in F_{n}} \inf\nolimits_{c \in \mathbb{R}} \mathrm{me}_{\mu_{n}}(f,c) \leq \epsilon$.

(2)$\Longrightarrow$(3). This follows from Lebesgue's dominated convergence theorem.

(3)$\Longrightarrow$(2). Due to Remark~\ref{remark:antitone}, we have $\Delta (\mathscr{D}) \geq (\alpha \wedge 1) \cdot (\mathrm{ObsDiam}(\mathscr{D};-\alpha) \wedge 1)$ for any geometric data set $\mathscr{D}$ and any $\alpha \geq 0$. Consequently, if $\lim_{n \to \infty} \Delta (\mathscr{D}_{n}) = 0$, then $\lim\nolimits_{n \to \infty} \mathrm{ObsDiam}(\mathscr{D}_{n};-\alpha) = 0$ for every $\alpha > 0$, as desired. \end{proof}

We conclude this section with a useful remark about monotonicity.

\begin{proposition}\label{proposition:monotonicity} $\Delta \colon
  (\mathscrbf{D},\preceq) \to ([0,1],\leq)$ is monotone. \end{proposition}

\begin{proof} If $\mathscr{D}_{0} = (D_{0},F_{0},\mu_{0})$ and $\mathscr{D}_{1} = (D_{1},F_{1},\mu_{1})$ are geometric data sets such that $\mathscr{D}_{0} \preceq \mathscr{D}_{1}$, then there is $\phi \colon D_{1} \to D_{0}$ with $F_{0} \circ \phi \subseteq \overline{F_{1}}$ and $\phi_{\ast}(\mu_{1}) = \mu_{0}$, whence \begin{align*}
	\mathrm{ObsDiam}(\mathscr{D}_{0};-\alpha) \, & = \, \sup \{ \mathrm{PartDiam}(f_{\ast}(\mu_{0}),1-\alpha) \mid f \in F_{0} \} \\
		& = \, \sup \{ \mathrm{PartDiam}(f_{\ast}(\phi_{\ast}(\mu_{1})),1-\alpha) \mid f \in F_{0} \} \\
		& = \, \sup \{ \mathrm{PartDiam}((f \circ \phi)_{\ast}(\mu_{1}),1-\alpha) \mid f \in F_{0} \} \\
		& \leq \, \sup \left\{ \mathrm{PartDiam}(f_{\ast}(\mu_{1}),1-\alpha) \left\vert \, f \in \overline{F_{1}} \right\} \right. \\
		& = \, \sup \{ \mathrm{PartDiam}(f_{\ast}(\mu_{1}),1-\alpha) \mid \, f \in F_{1} \} \\
		& = \, \mathrm{ObsDiam}(\mathscr{D}_{1};-\alpha)
\end{align*} for every $\alpha \geq 0$, which readily implies that $\Delta (\mathscr{D}_{0}) \leq \Delta (\mathscr{D}_{1})$. \end{proof}

\section{Intrinsic Dimension}\label{section:intrinsic.dimension}

Below we propose an axiomatic approach to intrinsic dimension of geometric data
sets (Definition~\ref{definition:dimension.function}), a modification of ideas from~\citet{Pestov1} suited for our setup.

\begin{definition}\label{definition:dimension.function} A map $\partial \colon {\mathscrbf{D}} \to [0,\infty]$ is called a \emph{dimension function} if the following hold: \begin{enumerate}
	\item[$(1)$] \emph{Axiom of concentration:} \\ A sequence $(\mathscr{D}_{n})_{n \in \mathbb{N}} \in {\mathscrbf{D}}^{\mathbb{N}}$ has the L\'evy property if and only if \begin{displaymath}
						\lim\nolimits_{n\to \infty} \partial (\mathscr{D}_{n}) \, = \, \infty .
					\end{displaymath}
	\item[$(2)$] \emph{Axiom of continuity:} \\ If a sequence $(\mathscr{D}_{n})_{n \in \mathbb{N}} \in {\mathscrbf{D}}^{\mathbb{N}}$ concentrates to $\mathscr{D} \in {\mathscrbf{D}}$, then \begin{displaymath}
				\partial (\mathscr{D}_{n}) \, \longrightarrow \, \partial (\mathscr{D}) \ \quad (n \to \infty ).
			\end{displaymath}
	\item[$(3)$] \emph{Axiom of feature antitonicity:} \\ If $\mathscr{D}_{0}, \mathscr{D}_{1} \in {\mathscrbf{D}}$ and $\mathscr{D}_{0} \preceq \mathscr{D}_{1}$, then $\partial (\mathscr{D}_{0}) \, \geq \, \partial (\mathscr{D}_{1})$.
	\item[$(4)$] \emph{Axiom of geometric order of divergence:} \\ If $(\mathscr{D}_{n})_{n \in \mathbb{N}} \in {\mathscrbf{D}}^{\mathbb{N}}$ is a L\'evy sequence, then $\partial (\mathscr{D}_{n}) \in \Theta \! \left(\Delta (\mathscr{D}_{n})^{-2}\right)$.\footnote{Given two functions $f,g \colon \mathbb{N} \to [0,\infty)$, we write $f(n) \in \Theta (g(n))$ if there exist $N \in \mathbb{N}$ and $C > c > 0$ with $cf(n) \leq g(n) \leq C f(n)$ for all $n \geq N$.}
\end{enumerate} \end{definition}

\begin{remark} Let $\partial \colon {\mathscrbf{D}} \to [0,\infty]$ be a dimension function and let $\mathscr{D} = (D,F,\mu) \in {\mathscrbf{D}}$. Then $\partial (\mathscr{D}) = \infty$ if and only if $\vert D \vert = 1$. This is by force of the axiom of concentration. \end{remark}

\begin{proposition}\label{proposition:dimension} The map $\partial_{\Delta} \colon {\mathscrbf{D}} \to [1,\infty], \, \mathscr{D} \mapsto \frac{1}{\Delta (\mathscr{D})^{2}}$ is a dimension function. \end{proposition}

\begin{proof} Clearly, $\partial_{\Delta}$ is well defined on
  ${\mathscrbf{D}}$, since $\Delta$ is invariant under isomorphisms of geometric
  data sets, that is, $\Delta (\mathscr{D}_{0}) = \Delta (\mathscr{D}_{1})$ for
  any pair of isomorphic geometric data sets $\mathscr{D}_{0} \cong
  \mathscr{D}_{1}$. Also, $\partial_{\Delta}$ satisfies the axiom of
  concentration by Proposition~\ref{proposition:observable.diameters.levy} and the axiom of
  continuity by Proposition~\ref{proposition:observable.diameters.continuous}. Due to Proposition~\ref{proposition:monotonicity}, $\Delta \colon (\mathscrbf{D},{\preceq}) \to ([0,1],{\leq})$ is monotone, whence $\partial_{\Delta}$ satisfies the axiom of feature antitonicity. By definition, $\partial_{\Delta}$ obviously satisfies the axiom of geometric order of divergence. \end{proof}

As argued by~\citet{Pestov1,Pestov2}, it is desirable for a reasonable notion of intrinsic dimension to agree with our geometric intuition in the way that the value assigned to the Euclidean $n$-sphere $\mathbb{S}_{n}$, viewed as a geometric data set, would be in the order~of~$n$. To be more precise, for any integer~$n \geq 1$, let us consider the $mm$-space $\mathscr{S}_{n} \defeq (\mathbb{S}_{n},d_{\mathbb{S}_{n}},\xi_{n})$ where $d_{\mathbb{S}_{n}}$ denotes the geodesic distance on $\mathbb{S}_{n}$ and $\xi_{n}$ is the unique rotation invariant Borel probability measure on $\mathbb{S}_{n}$.

\begin{lemma}\label{lem:64} $\Delta ((\mathscr{S}^{n})_{\bullet}) = \Delta ((\mathscr{S}^{n})_{\circ}) \in \Theta \left( \frac{1}{\sqrt{n}} \right)$. \end{lemma}

\begin{proof} Let $\gamma$ denote the standard Gaussian measure on $\mathbb{R}$, i.e., $\gamma$ is the Borel probability measure on $\mathbb{R}$ given by $\gamma (B) \coloneqq \tfrac{1}{\sqrt{2\pi}} \int_{\mathbb{R}} \chi_{B}(t) \exp \left(-\tfrac{t^{2}}{2}\right) \,\mathrm{d}t$ for every Borel $B \subseteq \mathbb{R}$. According to~\citet[Corollary~8.5.7]{Shioya17} and~\citet[Proposition~2.19]{ShioyaBook}, \begin{equation}\tag{$\ast$}\label{nice}
	\sqrt{n} \cdot \mathrm{ObsDiam}((\mathscr{S}_{n})_{\bullet};-\alpha) \, \longrightarrow \, \mathrm{PartDiam}(\gamma,1-\alpha) \qquad (n \, \longrightarrow \, \infty)
\end{equation} for every $\alpha \in (0,1)$. Moreover, by~\citet[Theorem~2.29]{ShioyaBook}, \begin{displaymath}
	\sqrt{n} \cdot \mathrm{ObsDiam}((\mathscr{S}_{n})_{\bullet};-\alpha) \, \leq \, \sqrt{\tfrac{n}{n-1}} \cdot 2\sqrt{2} \sqrt{- \log\left( \sqrt{\tfrac{2}{\pi}} \alpha \right) } \, \leq \, 4 \sqrt{- \log\left( \sqrt{\tfrac{2}{\pi}} \alpha \right) }
\end{displaymath} for all $n \in \mathbb{N}_{\geq 2}$ and $\alpha \in (0,1]$. Since $\int_{0}^{1} 4 \sqrt{- \log\left( \sqrt{\tfrac{2}{\pi}} \alpha \right) } \, \mathrm{d} \alpha < \infty$, we may apply Lebesgue's dominated convergence theorem to conclude that \begin{align*}
	\limsup_{n \to \infty} \sqrt{n} \cdot \Delta ((\mathscr{S}_{n})_{\bullet}) \, &\leq \, \limsup_{n \to \infty} \int_{0}^{1} \sqrt{n} \cdot \mathrm{ObsDiam}((\mathscr{S}_{n})_{\bullet};-\alpha ) \, \mathrm{d}\alpha \\
	&= \, \int_{0}^{1} \mathrm{PartDiam}(\gamma,1-\alpha) \, \mathrm{d}\alpha \, < \, \infty ,
\end{align*} which entails that $\Delta ((\mathscr{S}_{n})_{\bullet}) \in O \left( \frac{1}{\sqrt{n}} \right)$.\footnote{Given two functions $f,g \colon \mathbb{N} \to [0,\infty)$, we write $f(n) \in O (g(n))$ if there exist $N \in \mathbb{N}$ and $C > 0$ such that $f(n) \leq C g(n)$ for all $n \geq N$.} On the other hand, picking any $\alpha_{0} \in (0,1)$ with $\int_{\alpha_{0}}^{1} \mathrm{PartDiam}(\gamma,1-\alpha) \, \mathrm{d}\alpha > 0$, we infer from~\eqref{nice} and Remark~\ref{remark:antitone} that \begin{displaymath}
	\exists n_{0} \in \mathbb{N} \ \forall n \in \mathbb{N}_{\geq n_{0}} \ \forall \alpha \in [\alpha_{0},1) \colon \qquad \mathrm{ObsDiam}((\mathscr{S}_{n})_{\bullet};-\alpha) \, < \, 1 .
\end{displaymath} Combining this with~\eqref{nice} and Lebesgue's dominated convergence theorem, we see that \begin{align*}
	\liminf_{n \to \infty} \sqrt{n} \cdot \Delta ((\mathscr{S}_{n})_{\bullet}) \, &\geq \, \liminf_{n \to \infty} \sqrt{n} \int_{\alpha_{0}}^{1} \mathrm{ObsDiam}((\mathscr{S}_{n})_{\bullet};-\alpha ) \wedge 1 \, \mathrm{d}\alpha \\
	&= \, \liminf_{n \to \infty} \int_{\alpha_{0}}^{1} \sqrt{n} \cdot \mathrm{ObsDiam}((\mathscr{S}_{n})_{\bullet};-\alpha ) \, \mathrm{d}\alpha \\
	&= \, \int_{\alpha_{0}}^{1} \mathrm{PartDiam}(\gamma,1-\alpha) \, \mathrm{d}\alpha \, > \, 0 ,
\end{align*} which shows that $\frac{1}{\sqrt{n}} \in O (\Delta((\mathscr{S}_{n})_{\bullet}))$. Thus, $\Delta((\mathscr{S}_{n})_{\bullet}) \in \Theta \left( \frac{1}{\sqrt{n}} \right)$ as desired. Also, due to~\citet[Proof of Lemma~2.33]{ShioyaBook}, $\mathrm{ObsDiam}((\mathscr{S}_{n})_{\bullet}) = \mathrm{ObsDiam}((\mathscr{S}_{n})_{\circ})$ for all $\alpha \in (0,1)$ and $n \in \mathbb{N}_{\geq 1}$. Hence, $\Delta ((\mathscr{S}_{n})_{\circ}) = \Delta((\mathscr{S}_{n})_{\bullet}) \in \Theta \left( \frac{1}{\sqrt{n}} \right)$. \end{proof}

By force of the axiom of geometric order of divergence, we have the following.

\begin{corollary} If $\partial \colon {\mathscrbf{D}} \to [0,\infty]$ is a dimension function, then \begin{displaymath}
		\partial ((\mathscr{S}_{n})_{\bullet}), \, \partial ((\mathscr{S}_{n})_{\circ}) \, \in \, \Theta (n) .
\end{displaymath} \end{corollary}

We continue by showing that the dimension function from Proposition~\ref{proposition:dimension} is compatible with the order of direct powers of
metric measure spaces. For any $n \in \mathbb{N}_{\geq 1}$ and an $mm$-space
$\mathscr{X} = (X,d,\mu)$, let $\mathscr{X}^{n} \coloneqq
(X^{n},d_{n},\mu^{\otimes n})$ where $d_{n}(x,y)
\coloneqq \tfrac{1}{n} \sum_{i=1}^{n} d(x_{i},y_{i})$ for all $x,y \in X^{n}$.

\begin{lemma}\label{lemma:delta} For any $\mathscr{X} \in {\mathscrbf{M}}$ with $0 < \diam (\mathscr{X}) \leq 1$, \begin{displaymath}
	\Delta ((\mathscr{X}^{n})_{\bullet}) , \, \Delta ((\mathscr{X}^{n})_{\circ}) \, \in \, \Theta \left( \tfrac{1}{\sqrt{n}} \right) .
\end{displaymath} \end{lemma}

\begin{proof} Due to~\citet[Theorem~1.1]{OzawaShioya} and~\citet[Proposition~2.19]{ShioyaBook}, \begin{displaymath}
	\mathrm{ObsDiam}((\mathscr{X}^{n})_{\bullet};-\alpha) \, \leq \, 4 \sqrt{2 \log \tfrac{2}{\alpha}} \cdot \tfrac{1}{\sqrt{n}}
\end{displaymath} for all $n \in \mathbb{N}$ and $\alpha \in (0,1)$. Since \begin{align*}
	K \, &\defeq \, 4 \sqrt{2} \int_{0}^{1} \sqrt{\log \tfrac{2}{\alpha}} \, \mathrm{d}\alpha \, = \, 4\sqrt{2} \left( 2 \int_{\sqrt{\log 2}}^{\infty} \exp(-t^{2}) \, \mathrm{d}t + \sqrt{\log 2} \right) \, \in \, (0,\infty) ,
\end{align*} thus $\Delta ((\mathscr{X}^{n})_{\circ}) \leq \Delta ((\mathscr{X}^{n})_{\bullet}) \leq \frac{K}{\sqrt{n}}$ for all $n \in \mathbb{N}$. So, \begin{displaymath}
	\Delta ((\mathscr{X}^{n})_{\bullet}) , \, \Delta ((\mathscr{X}^{n})_{\circ}) \, \in \, O \! \left(\tfrac{1}{\sqrt{n}}\right) .
\end{displaymath} Conversely, the argument in~\cite[Proof of Theorem~1.3]{OzawaShioya}, together with~\cite[Proposition~2.19]{ShioyaBook}, asserts the existence of a positive real number $V(\mathscr{X})$ such that \begin{displaymath}
	\forall \alpha \in (0,1) \colon \ \ \liminf_{n \to \infty} \sqrt{n} \cdot \mathrm{ObsDiam}((\mathscr{X}^{n})_{\circ};-\alpha) \, \geq \, \sqrt{V(\mathscr{X})} \cdot \mathrm{PartDiam}(\nu, 1-\alpha) ,
\end{displaymath} where $\nu$ is the Borel probability measure on $\mathbb{R}$ given by \begin{displaymath}
	\nu (B) \, \defeq \, \sqrt{\tfrac{2}{\pi}} \int_{0}^{\infty} \chi_{B}(t) \exp \left(-\tfrac{t^{2}}{2}\right) \,\mathrm{d}t
\end{displaymath} for every Borel $B \subseteq \mathbb{R}$. Thus, thanks to Fatou's lemma and the fact that $\diam (\mathscr{X}) \leq 1$, \begin{align*}
	\liminf_{n \to \infty} \sqrt{n} \cdot \Delta ((\mathscr{X}^{n})_{\circ}) \, & \geq \, \liminf_{n \to \infty} \int_{0}^{1/2} \sqrt{n}\cdot \mathrm{ObsDiam}((\mathscr{X}^{n})_{\circ};-\alpha) \, \mathrm{d}\alpha \\
	& \geq \, \int_{0}^{1/2} \liminf_{n \to \infty} \sqrt{n}\cdot \mathrm{ObsDiam}((\mathscr{X}^{n})_{\circ};-\alpha) \, \mathrm{d}\alpha \\
	& \geq \, \sqrt{V(\mathscr{X})} \int_{0}^{1/2} \mathrm{PartDiam}(\nu, 1-\alpha) \, \mathrm{d}\alpha \\
	& \geq \, \tfrac{1}{2}\sqrt{V(\mathscr{X})} \cdot \mathrm{PartDiam}\left(\nu, \tfrac{1}{2}\right) \, \in \, (0,\infty ) ,
\end{align*} which implies that $\frac{1}{\sqrt{n}} \in O(\Delta ((\mathscr{X}^{n})_{\circ}))$, and so $\frac{1}{\sqrt{n}} \in O(\Delta ((\mathscr{X}^{n})_{\bullet}))$. It follows that $\Delta ((\mathscr{X}^{n})_{\bullet}), \Delta ((\mathscr{X}^{n})_{\circ}) \in \Theta \left( \frac{1}{\sqrt{n}} \right)$. \end{proof}

Again, we arrive at a geometric consequence for dimension functions.

\begin{corollary}\label{dimfunc} Let $\partial \colon {\mathscrbf{D}} \to
  [0,\infty]$ be a dimension function. For every $\mathscr{X} \in
  {\mathscrbf{M}}$ with  $0 < \diam (\mathscr{X}) \leq 1$,
  \[\partial_{\Delta} ((\mathscr{X}^{n})_{\bullet}) , \, \partial_{\Delta} ((\mathscr{X}^{n})_{\circ})\, \in\, \Theta (n).\] \end{corollary}

\section{Applications}
Equipped with this new notion of dimension function, we propose two
applications in the field of machine learning. The first is situated in a
classical learning realm where data sets are represented as subsets of
$\mathbb{R}^{n}$. The second applies to purely categorical data and the
challenges that arise with that.

\subsection{Distance-Based Machine Learning Methods}
\label{sec:dens-based-clust}
Distance functions are fundamental to the majority of ML
procedures. Classification tasks depend on this kind of features up to the same
proportion as clustering tasks do. Modeling distances as features of geometric data sets allows us to assign an intrinsic dimension to such problems and investigate its explanatory power for concrete real-world data.
So far there are only a few theoretical investigations of the dimension curse in
the realm of machine learning. One exception to this is the work of~\citet{Beyer99} investigating the impact of high
dimension in data to the kNN-Classification method. However, their main theoretical result~\citep[Theorem~1]{Beyer99} relies on a collection of assumptions rarely met by real-world data sets~\citep{Korn}. More recent works, e.g.,~\citet{Houle,Korn}, showed
that often the curse of dimensionality can be overcome through an appropriate
choice of feature functions. This illustrates the necessity to analyze data sets and
machine learning procedures based on their features. In the present section, we compute the dimension function established in Corollary~\ref{dimfunc} in order to detect and
quantify the extent of dimension curse in concrete data.

\subsubsection{Distances as Features}
\label{sec:dist-feat-data}
Let $n \in \mathbb{N}_{\geq 1}$ and let $d_{\mathrm{eucl}}$ denote the Euclidean metric on~$\mathbb{R}^{n}$. Given a non-empty finite subset $X\subseteq\mathbb{R}^{n}$ of points to be analyzed via some distance-based machine learning procedure, we propose to study the geometric data set \begin{displaymath}
	\mathscr{D}_{n}(X) \, \defeq \, \left(X,d_{\mathrm{eucl}}\vert_{X^{2}},\nu_{X}\right)_{\circ} \, = \, \left(X,\{ x \mapsto d_{\mathrm{eucl}}(x,y) \mid y \in X \},\nu_{X}\right) ,
      \end{displaymath} cf.~Definition~\ref{definition:induced.data.sets}. Furthermore, in order
      to be able to compare observable diameters of different data sets having
      different absolute diameters, we perform a normalization based on the following
      observation: for any geometric data set $\mathscr{D} = (Y,F,\mu)$ and
      $\alpha,\tau \geq 0$, it is not difficult to see that $\tau \cdot
      \mathrm{ObsDiam}(\mathscr{D};-\alpha)=\mathrm{ObsDiam}(\tau \cdot
      \mathscr{D};-\alpha)$, where\linebreak $\tau \cdot \mathscr{D} \defeq (Y,\{ \tau f \mid f \in F \},\mu)$. (The proof of the corresponding fact about $mm$-spaces is to be found in~\citet[Proposition~2.19]{ShioyaBook}) In particular, we may consider $\tau = \mathrm{diam}(Y,d_{F})^{-1}$ if $\vert Y \vert > 1$.

In Algorithms~\ref{algo:ObsDiam} and~\ref{algo:mindiam} we present a simple procedure for computing
the observable diameter of a geometric data set with distance features.  We may
infer from it an upper bound for the computational time complexity for computing
$\mathop{ObsDiam}$. Computing all features, i.e., all distances, requires
$O(cn^{2})$ time, where $c$ indicates the complexity for computing the distance
of two points in $X$. Computing the counting measure can be done alongside by
additionally counting the occurrence of a particular distance.  For every
distance we further have to compute the set of the minimal diameters. The
challenge here is traversing $f(X)$ for all possible subsets. Since the diameter
of some subset $B\subseteq f(X)$ is reflected by a choice of two points in $B$,
only subsets of cardinality two have to be checked, as shown
in Algorithm~\ref{algo:mindiam}, which requires
$O(n\cdot\sum\nolimits_{i=1}^{n}n-i)=O(n^{3})$ steps. The necessary time for
computing the maximum afterwards is subsumed by this. Hence, we conclude that
computing the observable diameter for a given geometric data set using distances
as features is at most in $O(cn^{2}+n^{3})$ for run-time complexity.

\subsection{Intrinsic Dimension of Incidence Geometries}
\label{sec:conc-clust}

As a second exemplary application of the intrinsic dimension function we choose
incidence structures as investigated in Formal Concept Analysis (FCA). These
data tables are natural in a way that they are widely used in data science far
beyond FCA. We recall the basic notions of FCA relevant to this work. For a
detailed introduction to FCA, we refer to~\citet{fca-book}. Let $\mathbb{K} =
(G,M,I)$ be a \emph{formal context}, i.e., a triple consisting of two non-empty
sets $G$ and $M$ and a relation $I \subseteq G \times M$. The elements of $G$
are called the \emph{objects} of $\mathbb{K}$ and the elements of $M$ are called
the \emph{attributes} of $\mathbb{K}$, while $I$ is referred to as the
\emph{incidence relation} of $\mathbb{K}$. We call $\mathbb{K}$ \emph{empty} if
$I = \emptyset$, and \emph{finite} if both $G$ and $M$ are finite.  For $A
\subseteq G$ and $B \subseteq M$, put
\begin{align*}
&A'\,\defeq\,\{ m \in M \mid \forall g \in A
\colon \, (g,m) \in I \}, &B'\,\defeq\,\{ g \in G \mid \forall m \in B \colon
\, (g,m) \in I \}.
\end{align*}
As common in formal concept analysis, we will refer to the
elements of
\begin{align*}
&\mathfrak{B}(\mathbb{K})\,\defeq\,\{ (A,B) \mid A \subseteq G, \, B
\subseteq M, \, A' = B, \, B' = A \}
\end{align*}
as \emph{formal concepts} of
$\mathbb{K}$. We endow $\mathfrak{B}(\mathbb{K})$ with the partial order given
by
\begin{align*}
(A,B)\,\leq\,(C,D)\quad :\Longleftrightarrow\quad A\,\subseteq\,C
\end{align*}
for  $(A,B),(C,D)\in\mathfrak{B}(\mathbb{K})$.

\subsubsection{Concept Lattices as Geometric Data Sets}
\label{application:geometric}
In order to assign an intrinsic dimension to a concept lattice, we need to
transform a formal context into a geometric data set accordant to
Definition~\ref{defi:ds}. The crucial step here is a meaningful choice for the set of
features, which should reflect essential properties for the applied machine
learning procedure, or employed knowledge discovery process. Holding on to this
idea, we propose the following construction.

\begin{definition} Let $\mathbb{K} = (G,M,I)$ be a finite formal context. The geometric data set \emph{associated to} $\mathbb{K}$ is defined to be $\mathscr{D}(\mathbb{K}) \defeq (M,F(\mathbb{K}),\nu_{M})$ with 
  \begin{align*}
    F(\mathbb{K})\,\defeq\,\{ \nu_{G}(A)\cdot \mathds{1}_{B} \mid (A,B) \in \mathfrak{B}(\mathbb{K})\}. 
  \end{align*}
\end{definition}

Let us unravel Definition~\ref{definition:observable.diameter} for data sets arising from formal contexts.

\begin{proposition}\label{prop:fca}
  Let $\mathbb{K} = (G,M,I)$ be a finite formal context and let $\alpha \geq 0$. For every concept $(A,B) \in \mathfrak{B}(\mathbb{K})$, \begin{displaymath}
	\mathrm{PartDiam}((\nu_{G}(A)\cdot \mathds{1}_{B})_{\ast}(\nu_{M}),1-\alpha) \, = \, \begin{cases}
			\nu_{G}(A) & \text{if } \alpha < \nu_{M}(B) < 1 - \alpha , \\
			0 & \text{otherwise.}
                      \end{cases}
\end{displaymath} Hence, $
  \mathrm{ObsDiam}(\mathscr{D}(\mathbb{K});-\alpha) \, = \, \sup \{ \nu_{G}(A) \mid (A,B) \in \mathfrak{B}(\mathbb{K}), \, \alpha < \nu_{M}(B) < 1 - \alpha \} .
$ \end{proposition}

Note that in the special case of an empty context the observable diameter of the
associated data set is zero, in accordance with Definition~\ref{definition:observable.diameter}.

\subsubsection{Intrinsic Dimension of Scales}
There are particular formal contexts used for scaling non-binary attributes into
binary ones. Investigating them increases the first grasp for the intrinsic
dimension of concept lattices.  The most common scales are the \emph{nominal
  scale}, $\mathbb{K}_{n}^{\mathrm{nom}} \defeq ([n],[n],=)$, and the
\emph{contranominal scale}, $\mathbb{K}_{n}^{\mathrm{con}} \defeq
([n],[n],\neq)$, where $[n] \defeq \{1,\dotsc,n\}$ for a natural number $n \geq
1$. A straightforward application of the trapezoidal rule reveals that
\begin{align*}
  \Delta(\mathscr{D}(\mathbb{K}_{n}^{\mathrm{con}}))=\int_{0}^{1/2}\mathrm{ObsDiam}(\mathscr{D}(\mathbb{K}_{n}^{\mathrm{con}});-\alpha)\,\mathrm{d}\alpha
  &= \tfrac{1}{n}\left(\tfrac{1}{2}\tfrac{n-1}{n}+\sum\nolimits_{k=1}^{n/2-1}\tfrac{n-k}{n}\right).
\end{align*}
So,
$\lim_{n\to\infty}\partial_{\Delta}(\mathscr{D}(\mathbb{K}_{n}^{\mathrm{con}}))=\frac{64}{9}$.
For the nominal scale, we see that
$\partial_{\Delta}(\mathscr{D}(\mathbb{K}_{n}^{\mathrm{nom}})) = n^{4}$, which
diverges to $\infty$ as $n\to\infty$. In the latter case, we observe that our
intrinsic dimension reflects the dimension curse appropriately as the number of
attribute increases.

\section{Conclusion}
\label{sec:conc}
This work provides a comprehensive approach to intrinsic dimensionality of a
data set, as often encountered explicitly or implicitly in machine learning and
knowledge discovery. Inspired by and extending Pestov's work, we introduced a space of geometric data sets, developed a natural axiomatization of intrinsic dimension, and established a specific dimension function satisfying the axioms proposed. Our axiomatic approach (hence every concrete instance) reflects the
dimension curse correctly and agrees with common geometric intuition in various respects. Furthermore,
it facilitates a quantification of the dimension curse. We illustrated our
feature-based approach through exemplary computations for various artificial and
real-world data sets. For those we observed a difference in evaluation by the 
intrinsic dimension function compared to Chavez intrinsic dimension.

We identify various future works. Due to the challenging task to compute the
intrinsic dimension, in particular in the case of incidence structures,
heuristics for approximation are of great interest. For example, one could apply
feature sampling. Furthermore, an important problem to be investigated is the
influence of feature selection or feature reduction, like principle component
analysis, to the value of intrinsic dimension, which should lead to a monotone
increase in the values of the intrinsic dimension.

\section*{Acknowledgements}

The authors would like to express their sincere gratitude to Vladimir
Pestov for a number of insightful comments on this work, as well as to
the anonymous referee for their very careful review of this
manuscript.

\appendix
\section{Experiments}

To motivate the use of our results we added two experimental
investigations to this work. The first is concerned with the distance based
learning approach as discussed in Section~\ref{sec:dens-based-clust}. The second
explores the proposed intrinsic dimension function with respect to incidence
geometries as treated in Section~\ref{sec:conc-clust}.

\subsection{Experiment: Distances as Features}
For this experiment we applied the algorithms as depicted
in Section~\ref{sec:naiveAlgo} to ten artificial and four real-world data sets. The
artificial ones in detail are:
\begin{inparadesc}
\item[Dimset$*$:] six data sets with 1024 data points in $\mathbb{R}^{d}$ for
  $d\in\{32, 64, 128, 256,512, 1024\}$, constructed and investigated
  in~\citep{DIMsets};
\item[Golf ball:] set of 4200 points resembling a three dimensional
  ball in $\mathbb{R}^{3}$ from~\citet{Ultsch2005};
\item[Wingnut:] 1,070 points resembling two antipodal dense rectangles in
  $\mathbb{R}^{2}$ from~\citet{Ultsch2005};
\item[Atom:] 800  points representing a golf ball containing a smaller golf
  ball, both having the same center coordinate in $\mathbb{R}^{3}$ from~\citet{Ultsch2005};
\item[Engy:] 4,096 points shaped in a circular and in an
  elliptical disc in $\mathbb{R}^{2}$ from~\citet{Ultsch2005}.
\end{inparadesc}
The four real-world data sets are in detail the following: 
\begin{inparadesc}
\item[Alon:] biological tumor data set that contains 2,000 measured gene
  expression levels of 40 tumor and 22 normal colon tissues from~\citet{Alon1999dy};
\item[Shippi:]  6,817 measured gene expression levels from 58 lymphoma
  patients from~\citet{Shipp}; 
\item[Nakayama:] 105 samples from 10 types of soft tissue tumors measured with
  22,283 gene expression levels from~\citet{Nakayama};
\item[NIPS:] the binary relation of 11463 words used in 5811 NIPS conference
  papers from~\citet{PerroneJST17}.
\end{inparadesc}

For comparison, alongside with the values of our dimension function
from Corollary~\ref{dimfunc}, we also computed the following quantity introduced by~\citep{Chavez}: given a non-void finite metric space $(X,d)$, let us refer to
\begin{align*}
\mathop{dim}\nolimits_{\textrm{dist}}(X)\,\coloneqq\,
\tfrac{\mu^{2}}{2\cdot\sigma^{2}}
\end{align*}
as the \emph{Chavez intrinsic dimension}, or simply \emph{Chavez ID}, of $(X,d)$, where $\mu \defeq \mathbb{E}_{\nu_{X2}}(d)$ is the expectation of $d$ with respect to $\nu_{X^{2}}$ and $\sigma \defeq \bigl( \mathbb{E}_{\nu_{X^{2}}}(d-\mu)^{2}\bigr)^{1/2}$ is the corresponding standard deviation.

\subsubsection{Observations}
\label{sec:evaluation}
\begin{table}
  \centering
    \caption{Intrinsic dimension for various data clustering sets.}
    {\footnotesize
    \begin{tabular}{@{}lrrrr@{}}\toprule
      Name&\# Points&\# Dimensions&Chavez ID&Intrinsic dimemsion\\\midrule
      dimset32&1,024&32&6.67&24.0\\
      dimset64&1,024&64&7.31&41.2\\
      dimset128&1,024&128&7.56&56.5\\
      dimset256&1,024&256&7.59&76.6\\
      dimset512&1,024&512&7.60&102.6\\
      dimset1024&1,024&1,024&7.59&116.2\\\midrule
      Golfball&4,200&3&4.00&8.89\\
      Wingnut&1,070&2&1.91&8.02\\
      Atom&800&3&1.45&11.0\\
      Engy&4,096&2&1.79&18.0\\\midrule
      Alon&62&2,000&3.50&13.9\\
      Shippi&58&6,817&4.12&36.9\\
      Nakayama&105&22,283&2.08&43.3\\
      NIPS&11,463&5,812&0.36&1463.6\\\bottomrule
    \end{tabular}}
  \label{tab:idclustering}
\end{table}

We illustrated the computational results of our algorithm for the featured data
sets in Figure~\ref{fig:dimsetplot}, and show the values for intrinsic
dimension (\emph{ID}) in Table~\ref{tab:idclustering}. For comparison we included the
values for the Chavez' intrinsic dimension (\emph{CID}).  Our first observation
is the repeating descend-pattern for the $\mathrm{ObsDiam}$-values of the dimset
data sets as shown in Figure~\ref{fig:dimsetplot}. We attribute this to the (unknown)
generation process for these data sets. The CID does not vary for the dimset
data sets with more than 64 dimensions, as depicted
in Table~\ref{tab:idclustering}. The interpretation for this drawn from~\citet{Chavez}
would be that the similarity between the points does not change when increasing
the number of dimensions. One would expect here that the intrinsic dimension
would stay constant as well. However, the intrinsic dimension increases
monotonously as the number of dimensions
goes 
to 1024.
\begin{figure}
  \centering
  \includegraphics[width=0.78\textwidth]{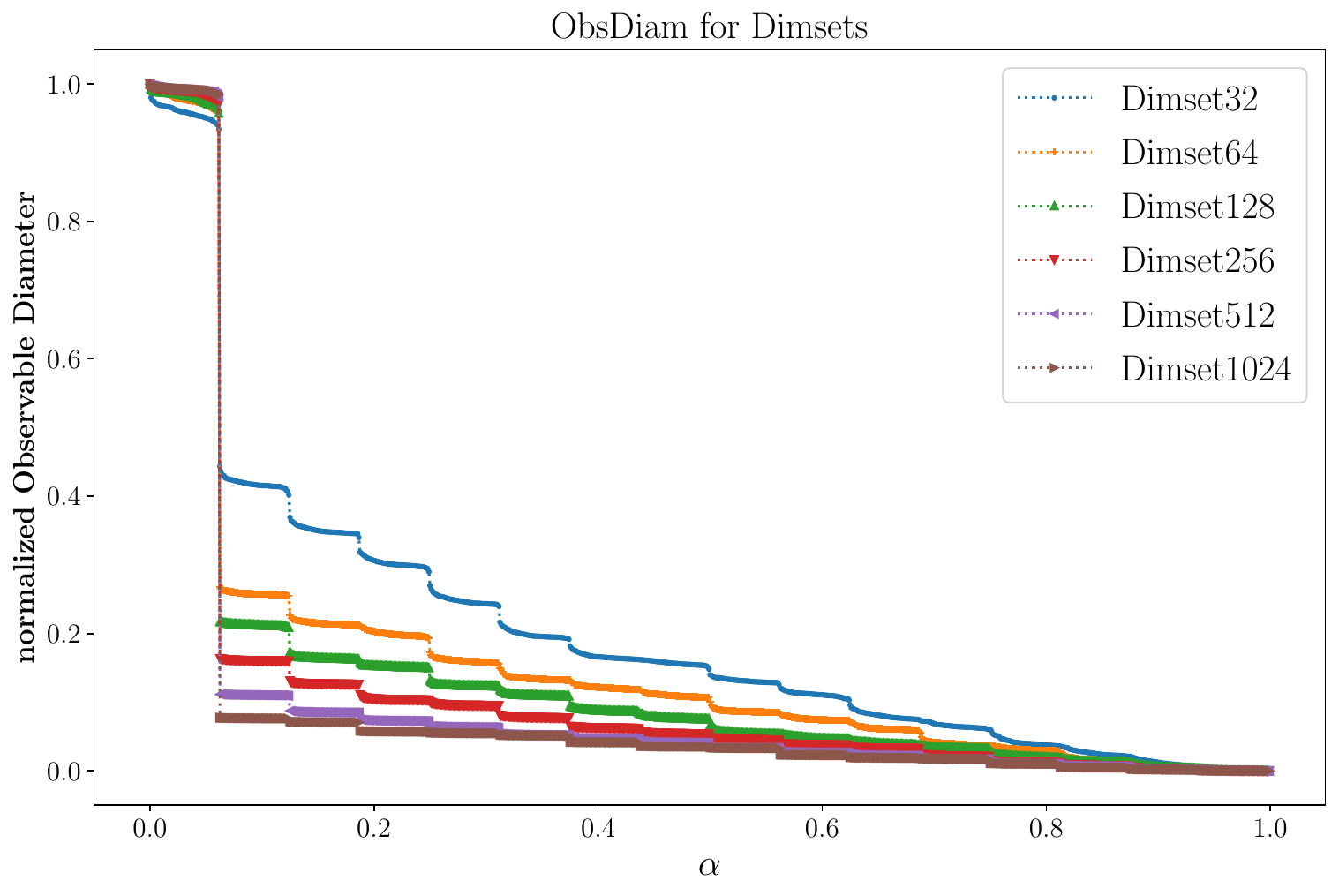}
    \includegraphics[width=0.8\textwidth]{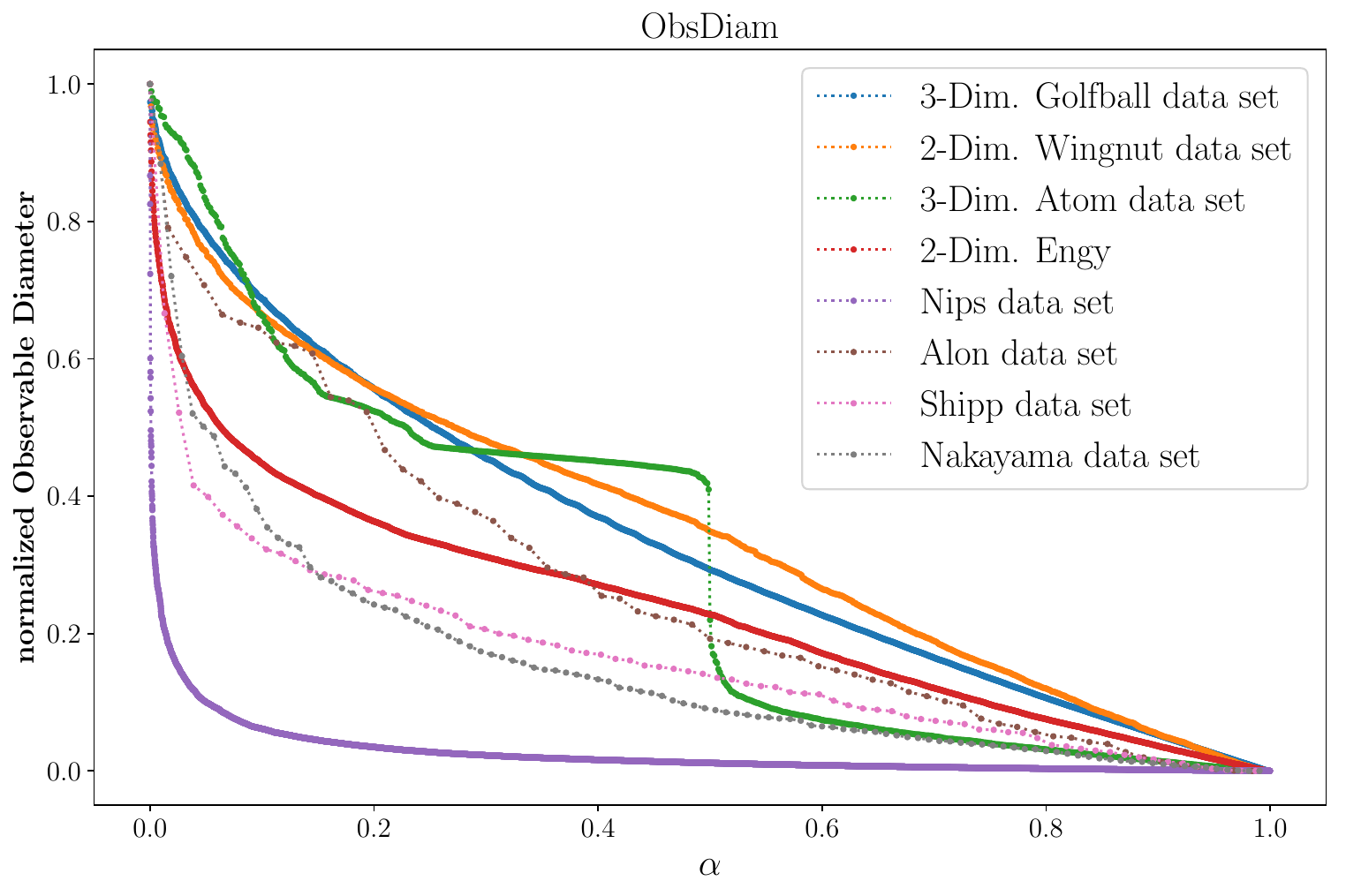}
  \caption{Observable diameter for $\alpha\in [0,1]$
    artificial data sets dimset (top) and data sets from~\citet{Ultsch2005} (bottom).}
  \label{fig:dimsetplot}
\end{figure}
Since all dimset data sets were generated using the same procedure with
the same number of point samples (1024) one would expect this increase. This is
not a mere correlation to the increase in the number of dimensions, but evidence
for the inability of the particular generation process to bound the intrinsic
dimension.  As for the low dimensional artificial data sets we observe a
different interaction between the CID an the ID. For example, the CID does
decrease when comparing the Golfball data set with the Atom data set, whereas
the intrinsic dimension increases. This indicates that the different dimension
functions cover different data set properties.

Finally, we compare the results for the real-world data sets. Even though the
number of dimensions is quite large, for those we may point out that the number
of point samples is quite small, in comparison. Nonetheless, all data sets have
essentially enough points to possibly span subspaces of 62 (Alon), 58 (Shippi),
and 105 (Nakayama) dimensions. We observe again that an increase in CID does not
precede an decrease in ID, as seen for Alon and Shippi. The converse, however,
can be observed as well when comparing Alon with Nakayama. The NIPS data set
exhibits by far the lowest CID as well as the highest ID. All these observations
lead us to conclude that the notion for intrinsic dimension, as introduced in
this work, captures an aspect of geometric data sets which is qualitatively
different to the Chavez intrinsic dimension.

\subsection{Experiments: Incidence Geometries}
\label{sec:comp-exper}

\begin{table}[t]
  \centering
  \caption{Intrinsic dimension for various data sets and their randomized
    counterparts.}
  {\footnotesize
  \begin{tabular}{@{}lrrrrr@{}}
    \toprule
    Name&\# Objects& \# Attributes& Density&\# Concepts& $\partial_\Delta(\mathscr{D}(\mathbb{K}))$\\\midrule
    zoo&101&28&0.30&379&52.44\\
    zoor&101&28&0.30&3339&1564.40\\\midrule
    cancer&699 & 92 & 0.10 & 9862 & 614.35\\
    cancerr&699 & 92 & 0.10 & 23151 & 417718.62\\\midrule
    southern&18 & 14 & 0.35 & 65 & 54.93\\
    southernr&18 & 14 & 0.37 & 120 & 167.01\\\midrule
    aplnm&79 & 188 & 0.06 & 1096 & 11667.14\\
    aplnmr&79 & 188 & 0.06 & 762 & 185324.01\\\midrule
    club &25& 15 & 0.25 & 62 & 118.15\\
    clubr&25 & 15 & 0.25 & 85 & 334.62\\\midrule
    facebooklike&377 & 522 & 0.01 & 2973 & 2689436.00\\
    facebookliker&377 & 522 & 0.01& 1265 & 5.73E7\\\midrule
    mushroom&8124&119&0.19& 238710& 263.49\\\bottomrule
  \end{tabular}}
  \label{tab:idtab}
\end{table}

We computed the intrinsic dimension function for different real-world data sets
to provide a first impression of
$\partial_{\Delta}(\mathscr{D}(\mathbb{K}))$. For brevity we reuse data sets
investigated by~\citet{Borchmann2017} and refer the reader there for an
elaborate discussion of those. All but one of the data sets are scaled versions
of downloads from the UCI Machine Learning Repository~\citep{Lichman:2013}. In
short we will consider the Zoo data set \emph{(zoo)} describing 101 animals by
fifteen attributes. The Breast Cancer data set \emph{(cancer)} representing 699
clinical cases of cell classification. The Southern Woman data set
\emph{(southern)}, a (offline) social network consisting of fourteen woman
attending eighteen different events. The Brunson Club Membership Network data
set \emph{(club)}, another (offline) social network describing the affiliations
of a set of 25 corporate executive officers to a set of 40 social
organizations. The Facebook-like Forum Network data set \emph{(facebooklike)}, a
(online) social network from an online community linking 377 users to 522
topics. A data set from an annual cultural event organized in the city of Munich
in 2013, the so-called \emph{Lange Nacht der Musik} \emph{(aplnm)}, a
(online/offline) social network linking 79 users to 188 events. And, finally the
well-known \emph{Mushroom} data set, a collection of 8124 described by 119
attributes. Additionally we consider for all those data sets, with exception for
mushroom, a randomized version. Those are indicated by the suffix \emph{r}.  We
conducted our experiments straightforward applying Proposition~\ref{prop:fca}. This was
done using
\texttt{conexp-clj}.\footnote{\url{https://github.com/exot/conexp-clj}} The
intermediate results for $\mathrm{ObsDiam}$ can be seen in
Figures~\ref{fig:obsdiam} and~\ref{fig:obsdiam2} and the final result for
$\partial_{\Delta}(\mathscr{D}(\mathbb{K}))$ is denoted in Table~\ref{tab:idtab}.
\begin{figure}
  \centering
  \includegraphics[width=1\textwidth]{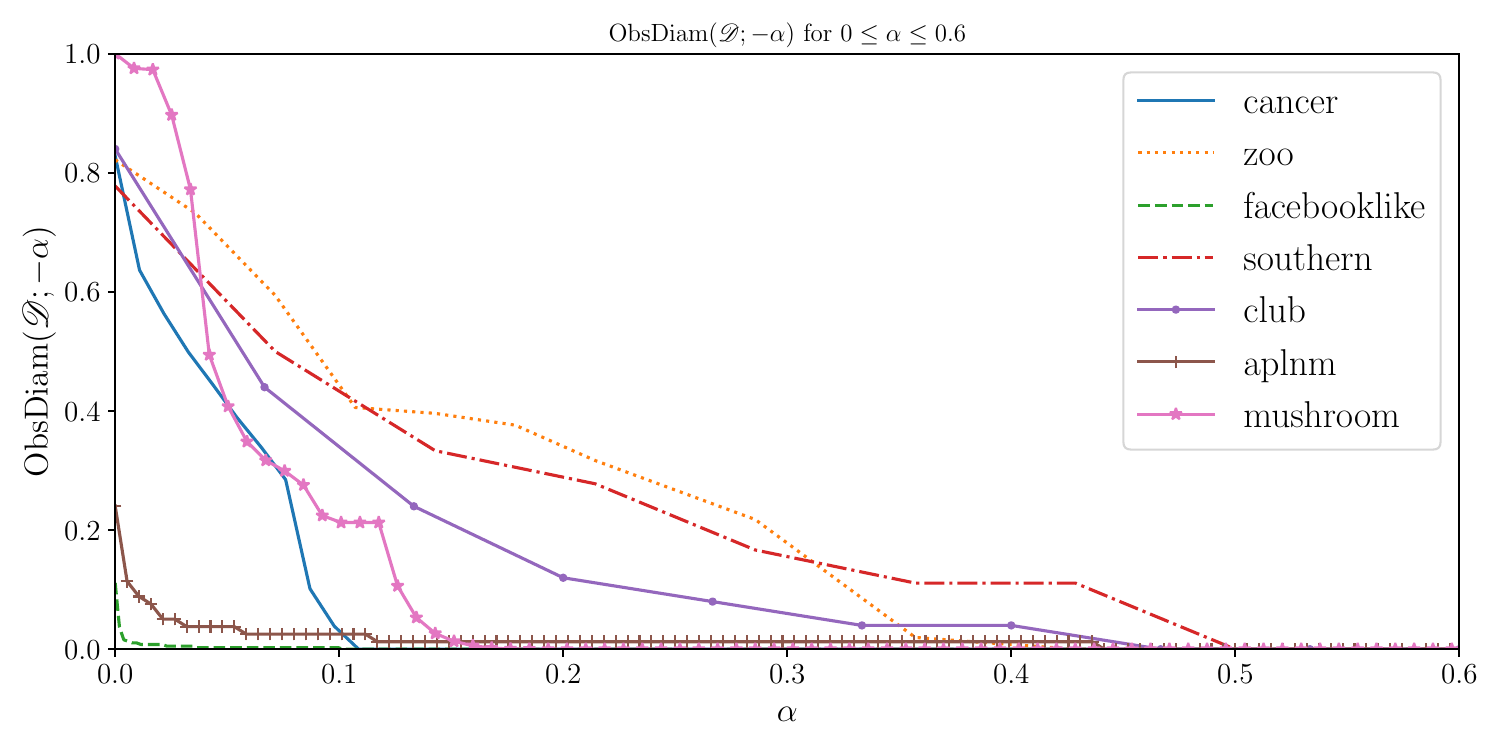}
  \caption{$\mathrm{ObsDiam}$ for all considered real-world data sets.}
  \label{fig:obsdiam}
\end{figure}

\begin{figure}
  \centering
  \includegraphics[width=1\textwidth]{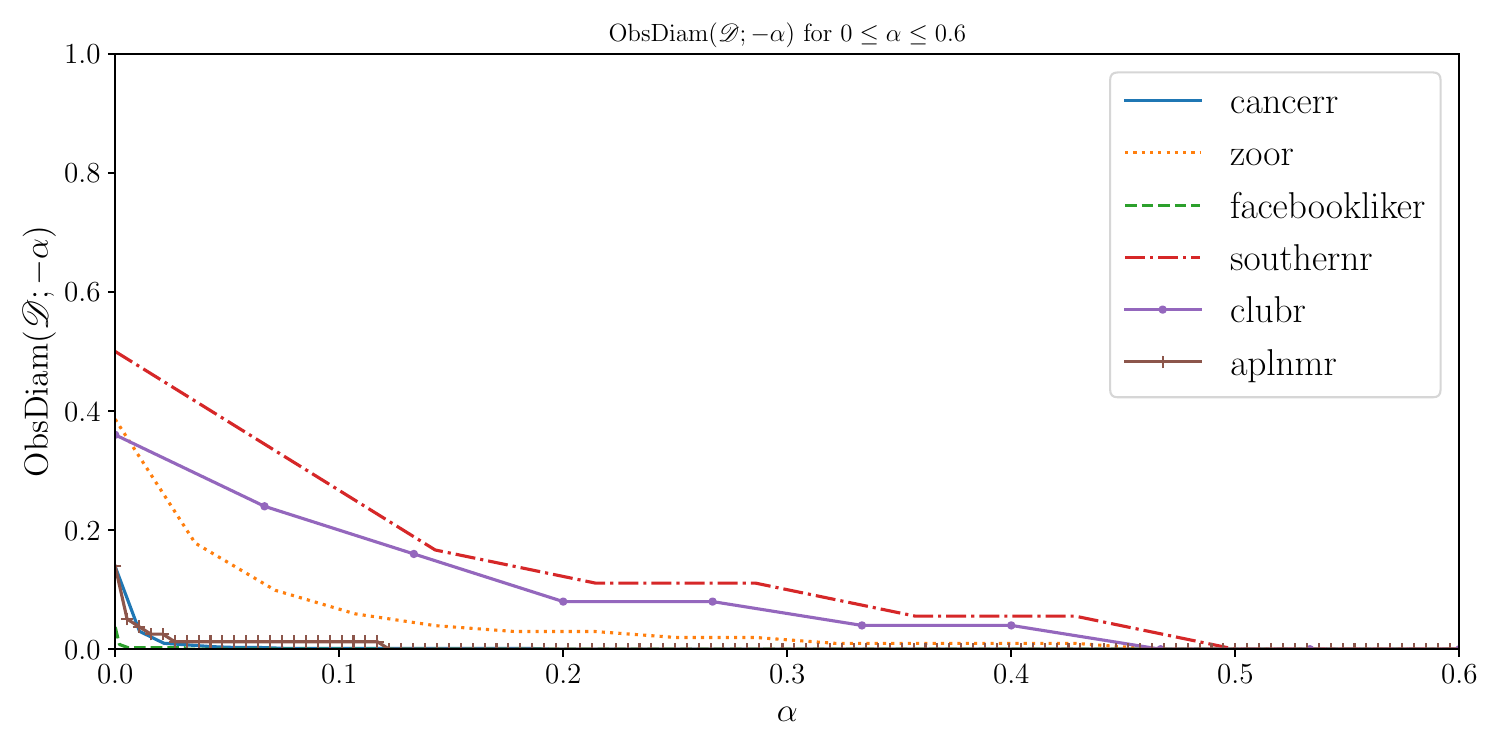}
  \caption{$\mathrm{ObsDiam}$ for randomized data sets based on Figure~\ref{fig:obsdiam}.}
  \label{fig:obsdiam2}
\end{figure}

\subsubsection{Observations}
All curves in Figure~\ref{fig:obsdiam} show a different behavior resulting in
different values for $\partial_{\Delta}(\mathscr{D})$. The overall descending
monotonicity is expected, however, the average as well as the local slopes are
quite distinguished. The general trend that comparably sparse contexts receive a
higher intrinsic dimension is also expected taking the results for the empty
context into account as well as the overall motivation of the curse of
dimension. Considering the random data sets in Table~\ref{tab:idtab} we observe that
neither the density nor the number of formal concepts (features) is an indicator
for the intrinsic dimension. This suggests that introduced intrinsic dimension
is independent of the usual descriptive properties. Comparing these results
to the Chavez ID is not applicable due to the non-metric nature of the
investigated data sets.

\section{Algorithms}
 \label{sec:naiveAlgo}
 \begin{algorithm}[H]
  \caption{$\mathrm{ObsDiam}$ with distance features}
  \label{algo:ObsDiam}
\begin{lstlisting}[mathescape,basicstyle=\footnotesize]
define ObsDiam($X,\mu,F$)  (*@ \hfill @*)   ; returns List
  for $f$ in $F$: (*@\label{a:distanzen}@*)
    $V_{f}=\{\}$
    Measure ={:} (*@ \hfill @*)   ; dictionary for measures
    for $x$ in $X$:
      $V_f = V_{f}\cup\{f(x)\} $
      Measure[$f(x)$] =+ 1  (*@ \hfill @*)   ; preimage measure increase
    matrix[f,:] = MinDiamMatrix($V_f$,Measure,X)
  for $\alpha$ in $ (0,1/|X|,\ldots,(|X|-1)/|X|,1)$
    result[$\alpha$]=$\max($matrix[:,$\alpha$]$)$
  return result
\end{lstlisting}
\end{algorithm}
\begin{algorithm}[H]
  \caption{MinDiamMatrix with distance features}
  \label{algo:mindiam}
\label{algo:partialDiam}
\begin{lstlisting}[mathescape,basicstyle=\footnotesize]
define MinDiamMatrix($V_{f},\text{Measure},X$)  (*@ \hfill @*)   ; returns Matrix
    result = $(\diam(X),\ldots,\diam(X))$  (*@ \hfill @*) ; Initialize vector with length $|X|$
    for s in $(V_{f},\leq)$: (*@ \hfill @*) ; iterate through $V_{f}$
      my_of_x = Measure[s] $\cdot |X|$   (*@ \hfill @*) ; denormalization to get index
      diam_of_x = 0
      if  result[my_of_x] > diam_of_x then result[my_of_x] = diam_of_x
      for e in $\{d\in V_{f}\mid d \geq s\}\leq V_{f}$
            my_of_x =+ Measure[e] $\cdot|X|$
            diam_of_x = e - s
            if  result[my_of_x] > diam_of_x then
              result[my_of_x] = diam_of_x
    for i in $(|X|,|X|-1,\ldots,1)$:   (*@ \hfill @*)   ; repair monotonicity if necessary
        if result[i] < result[i - 1] then result[i - 1] = result[i]
    return result
\end{lstlisting}
\end{algorithm}

\DeclareDelimFormat{finalnamedelim}{\addspace\bibstring{and}\space}
\DeclareFieldFormat[article]{pages}{#1}
\renewbibmacro*{volume+number+eid}{%
  \printfield{volume}
  \setunit{\addcomma\space}%
  \printfield{eid}}

\renewbibmacro*{number+date}{%
  \printfield[parens]{year}
  \setunit{\addcomma\space}
  \iffieldundef{number}
  {}
  {
    \newunit\printtext{no\adddot}
    \printfield{number}
  }
}

\renewbibmacro*{journal+issuetitle}{%
  \printfield{journaltitle}
  \iffieldundef{series}
    {}
    {\newunit
     \printfield{series}%
     \setunit{\addspace}}%
  \usebibmacro{volume+number+eid}%
  \setunit{\addspace}%
  \usebibmacro{number+date}%
  \newunit}
\renewcommand*{\mkbibnamefamily}[1]{\textsc{#1}}
\renewcommand*{\mkbibnameprefix}[1]{\textsc{#1}}
\DeclareFieldFormat*{title}{#1}
\renewcommand*{\newunitpunct}{\addcomma\space}
\renewbibmacro{in:}{}
\printbibliography{}

\end{document}